\documentclass[lettersize,journal]{IEEEtran}
\usepackage[utf8]{inputenc}

\usepackage{amsmath, amssymb, amsfonts, bm, amsthm}

\usepackage{algorithm}
\usepackage{algorithmic}

\usepackage{graphicx}
\usepackage{epstopdf}      % EPS 转 PDF
\usepackage{caption}
\usepackage[caption=false,font=normalsize,labelfont=sf,textfont=sf]{subfig}  % 推荐用于 IEEE
\usepackage{threeparttable}  % 表格脚注
\usepackage{booktabs}        % 表格横线
\usepackage{makecell}        % 表格内换行
\usepackage{tabularx}        % 自动列宽
\usepackage{array}           % 自定义列格式
\usepackage{multirow}
\usepackage{float}           % 使用 [H]
\usepackage{adjustbox}       % 图表缩放与对齐
\usepackage{xpatch}

\usepackage[table,xcdraw]{xcolor}
\usepackage{soul}            % \hl{}
\usepackage{colortbl}

\usepackage{url}
\usepackage{hyperref}
\usepackage{breakurl}
\usepackage{cite}

\usepackage{indentfirst}
\usepackage{textcomp}
\usepackage{flushend}        % 最后一页双栏对齐
\usepackage{cuted}           % 跨栏内容（如 \twocolumn[]）
\usepackage{afterpage}       % 精细控制浮动体

\theoremstyle{remark}
\newtheorem{lemma}{Lemma}
\newtheorem{assumption}{Assumption}

\newcommand{\changed}[1]{\textcolor{black}{#1}}

\newcolumntype{Y}{>{\centering\arraybackslash}X}
\renewcommand{\arraystretch}{1.2}

\usepackage{relsize}
\AtBeginEnvironment{tabular}{\smaller}
\AtBeginEnvironment{tabular*}{\smaller}
\AtBeginEnvironment{tabularx}{\smaller}

\makeatletter
\xpatchcmd{\proof}
  {\topsep6\p@\@plus6\p@\relax}
  {}{}{}
\makeatother

\title{LEMON-Mapping: Loop-Enhanced Large-Scale Multi-Session Point Cloud Merging and Optimization for Globally Consistent Mapping}

\author{
        Lijie Wang$^{1}$,
        Xiaoyi Zhong$^{1}$,
        Ziyi Xu$^{1}$,
        Kaixin Chai$^{2}$,
        Anke Zhao$^{1}$,
        Tianyu Zhao$^{1}$,
        Changjian Jiang$^{3}$,
        Qianhao Wang$^{1,\dag}$,
        Xieyuanli Chen$^{4, \dag}$
	      Fei Gao$^{1, \dag}$
\thanks{$^1$Institute of CyberSystems and Control, Zhejiang University, Hangzhou 310027, China.  (e-mail:
3210101760@zju.edu.cn)
}% <-this % stops a space
\thanks{$^2$The Huzhou Institute, Zhejiang University, Huzhou 313000, China.}
\thanks{$^3$The State Key Laboratory of Industrial Control Technology, College of Control Science and Engineering, Zhejiang University, Hangzhou 310027, China.}
\thanks{$^4$The College of Intelligence Science and Technology, National University of Defense Technology, Changsha 410007, China.}
\thanks{$^{\dag}$Corresponding author: Qianhao Wang, Xieyuanli Chen and Fei Gao. Email:  fgaoaa@zju.edu.cn.}

}
\vspace{0.5cm}

\makeatletter
\let\@oldmaketitle\@maketitle% Store \@maketitle
\renewcommand{\@maketitle}
{
\@oldmaketitle % Update \@maketitle to insert...
\centering
\setcounter{figure}{0}
\begin{minipage}{1.0\linewidth}
	\includegraphics[width=1.0\textwidth]{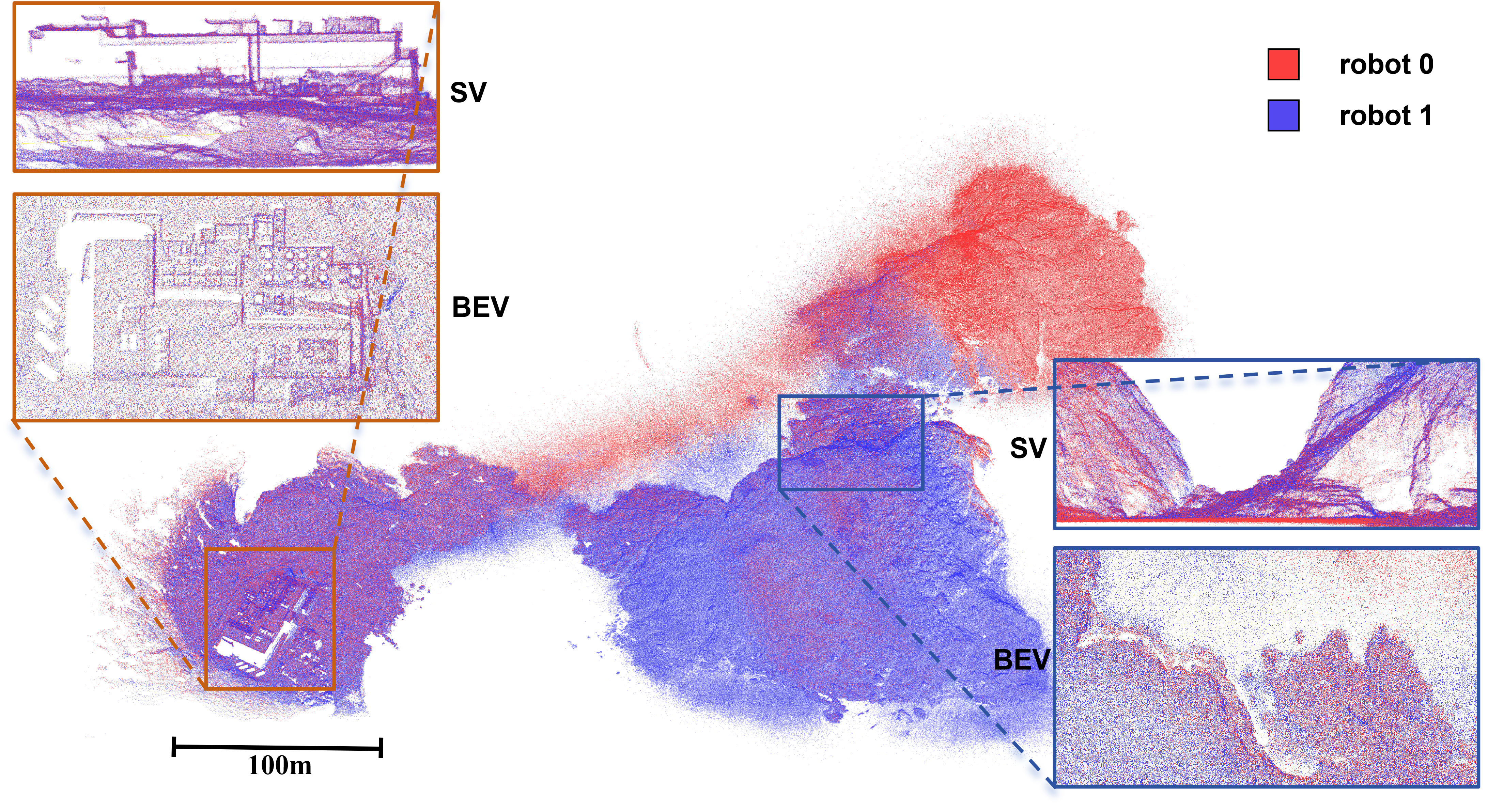}
	\vspace{-0.3cm}
	\captionof{figure}{The \changed{map merging result} of our framework in island sequence of MARS-LVIG \cite{lvig} dataset, the details in the figure are framed
and shown in two forms: side view (SV) and bird’s eye view (BEV).}
	\label{fig:headimg}
\end{minipage}
\vspace{-0.5cm}
}
\makeatother
\makeatletter

\def\ntpname{Note to Practitioners}

\def\ntp{\normalfont
    \if@twocolumn
      \@IEEEabskeysecsize\bfseries\textit{\ntpname}---\relax
    \else
      \bgroup\par\addvspace{0.5\baselineskip}\centering\vspace{-1.78ex}%
      \@IEEEabskeysecsize\textbf{\ntpname}\par\addvspace{0.5\baselineskip}\egroup
      \quotation\@IEEEabskeysecsize
    \fi\@IEEEgobbleleadPARNLSP}

\def\endntp{\relax\ifCLASSOPTIONconference\vspace{0ex}\else\vspace{1.34ex}\fi\par
    \if@twocolumn\else\endquotation\fi
    \normalfont\normalsize}

\makeatother

\begin{document}
\maketitle

\begin{abstract}
Multi-robot collaboration is critical but challenging for building globally consistent maps. Traditional multi-robot pose graph optimization (PGO) methods ensure basic global consistency but ignore the geometric structure of the map and only use loop closures as constraints between pose nodes, which leads to divergence and blurring in overlapping regions. To address this, we propose LEMON-Mapping, a loop-enhanced framework for large-scale, multi-session point cloud fusion and optimization. We re-examine the role of loops in multi-robot mapping and present three key innovations. First, we develop a robust loop processing mechanism with outlier rejection and a recall strategy to recover valid loops. Second, we introduce spatial bundle adjustment to reduce divergence and eliminate blurring in overlapping areas. Third, we design a PGO-based optimization that integrates refined bundle adjustment constraints to propagate local accuracy globally. Experiments on public and self-collected datasets demonstrate that LEMON-Mapping achieves superior accuracy, consistency, and scalability over traditional approaches in large-scale multi-robot scenarios.
\end{abstract}

\begin{ntp}
In this paper, we address inaccurate and inconsistent large-scale multi-robot map fusion, where conventional multi-robot SLAM ignores geometric structures, causing divergence and blurring in overlapping regions.
In real-world exploration, multi-robot mapping is designed to extend coverage, but inconsistent and inaccurate map fusion degrades map quality and affects downstream applications like re-localization, navigation, and robotic operation, potentially leading to navigation failures and unsafe behaviors. LEMON-Mapping integrates robust loop processing, spatial bundle adjustment, and pose graph optimization to reconstruct accurate and globally consistent multi-robot maps, providing reliable support for multi-robot collaboration, localization, and autonomous navigation.    
\end{ntp}


%% introduction
%%%%%%%%%%%%%%%%%%%%%%%%%%%%%%%%%%%%%%%%%%%%%%%%%%%%%%%%%%%%%%%%%%%%%%%%%%%%%%%%
\section{Introduction}
% importance
%传统应用
\IEEEPARstart{L}ARGE-SCALE 3D mapping is a fundamental capability in modern robotics, providing rich geometric information that supports tasks such as drone-based inspection \cite{SRFlight, swarmlio, swarmlio2}, autonomous driving \cite{driving}, and long-term exploration with ground robots \cite{rogmap}.
% 具身智能
Furthermore, large-scale 3D maps are critical for emerging fields such as embodied AI, where agents interact with complex environments based on spatial understanding \cite{iplanner}. They also play a key role in end-to-end visuomotor navigation systems \cite{viplanner}, which rely on accurate environmental priors to enhance generalization. 
Especially, multi-robot 3D mapping is essential for large-scale and complex tasks, where robot teams provide greater coverage and robustness compared to a single agent. In scenarios such as search and rescue \cite{uavrescue}, forest monitoring \cite{forest}, and subterranean exploration \cite{darpa}, accurate point cloud fusion is vital to enable cooperation among multiple robots under complex and GPS-denied conditions.

% LiDAR-SLAM, 两点问题(1) 不注重建图
Rapid development of LiDAR-based Simultaneous Localization and Mapping (SLAM) techniques \cite{loam, slict, fastlio2, fasterlio, dlio, pointlio, slim, lta-om, voxel-slam} has enabled individual robots to perform real-time localization and point cloud map construction. However, most of the existing lightweight SLAM systems \cite{fastlio2, fasterlio, dlio} for onboard computers are based on direct LiDAR-Inertial-Odometry methods. These methods prioritize pose estimation over map construction to achieve real-time performance. However, the lack of feature extraction and matching may result in maps with blurred details and poor geometric accuracy. Therefore, the quality of the point cloud maps reconstructed by onboard computers based on lightweight methods requires further improvement.

% Second, these systems are typically designed for single-robot settings and can not directly be used in multi-robot scenarios, thus limiting the scalability and effectiveness of collaborative large-scale mapping.

%单机地图优化方法
%PGO
\changed{To improve the geometric accuracy and quality of point cloud maps, various strategies have been proposed.} Among them, one category of approaches leverages pose graph optimization (PGO) \cite{g2o} to handle loop scenarios. Loop closure constraints detected by loop detection methods \cite{std,btc,ring++,overlapnet} are added to the pose graph along with odometry constraints to eliminate accumulated drift and improve global consistency of the map. However, PGO-based methods only use the point cloud map to extract descriptors during loop detection and ignore the structural and geometric information. \textbf{\textit{Therefore, although PGO-based methods can improve global consistency based on loop closures, it cannot guarantee an accurate map without divergence and blurring.}}
%BA
In contrast, another type of map refinement method based on bundle adjustment (BA)
has demonstrated impressive capabilities to improve the accuracy of point cloud maps in recent years \cite{balm1, balm2, hba, pss-ba, BALM3.0}. These approaches utilize the structural characteristics of the point cloud map and improve its quality by minimizing the geometric residuals between feature points and the feature structure.
\textbf{\textit{Although BA-based methods are capable of producing high-quality maps, their reliance on temporal sequences may prevent them from fully leveraging loop closure information, potentially resulting in poor map consistency.}} \changed{This dependence on the sequential poses also limits the applicability of BA-based methods to be directly used in multi-robot systems.}

%多机SLAM 与建图
In addition to the above mapping approaches, to achieve larger-scale map reconstruction and multi-robot collaboration, several multi-robot SLAM systems \cite{dclslam, discoslam} and offline map fusion techniques \cite{LAMM, automerge, pcm1, pcm2, pcm3, lidarvggt} have shown promising results in recent years. In general, these methods focus on two aspects: removing outliers in loop closure data and achieving map fusion through multi-robot PGO. However, PGO-based methods are fundamentally the same in single-robot and multi-robot applications, as they both use odometry constraints and loop closure constraints for map refinement. As discussed in the previous paragraph, these approaches neither directly utilize the geometric information in multi-robot maps nor fully leverage loop closure data.
Consequently, they only ensure the basic alignment and consistency of the global map but fail to \changed{reconstruct} high-quality point cloud maps. In particular, they often suffer from misalignment in regions with overlap between robots. \changed{Recently, several life-long SLAM systems \cite{voxel-slam, slim} handle multi-session mapping via a PGO-then-full-BA paradigm, where PGO and global BA are loosely coupled. To maintain controllable time costs, this pipeline relies on data sparsification or aggregation. Consequently, it is sensitive to initial PGO priors and often fails to achieve divergence-free and globally consistent map merging in multi-robot scenarios.}

%我们对原来方法的分析
\changed{Ultimately, whether applied independently or combined in current loosely coupled paradigms, existing PGO and BA methods often force a trade-off between high-resolution map accuracy and global consistency in multi-robot scenarios.} Building upon the strengths of existing methods and aiming to address their limitations, we perform an in-depth analysis of previous multi-robot SLAM systems and map merging techniques. The key issue of multi-robot mapping is to solve the \textbf{\textit{divergence}} and \textbf{\textit{blurring}} problems of submaps, while maintaining the \textbf{\textit{global consistency}} of the whole map.
We argue that the merging of point cloud maps is essentially a map-driven local registration problem rather than a pose-dominated pose optimization problem.
Building upon this perspective, we identify two essential challenges for achieving effective multi-robot point cloud map fusion: \textbf{(1)} Map fusion should be treated as a map-driven registration problem rather than a pose-centric optimization task such as PGO. \textbf{(2)} Loop closure information should be fully exploited not only to ensure global consistency but also to improve local accuracy.

%提出我们的framework
To overcome the above challenge, we propose \textbf{LEMON-Mapping}, a \textbf{L}oop-\textbf{E}nhanced large-scale \textbf{M}ulti-session map merging and \textbf{O}ptimizatio\textbf{N} framework that achieves globally consistent and geometrically accurate point cloud mapping.
%体现loop-enhanced, 发挥回环作用
Our framework re-examines the ability of loop closure and reasonably enhances its utilization by two aspects.
%回环召回
\changed{First, we introduce an innovative loop recall mechanism that provides a more comprehensive set of geometric constraints and enhanced optimization opportunities for the subsequent spatial BA.}
%体现空间BA对于多机的作用和保证局部准确性
Second, by leveraging sufficient valid loop closures, our spatial BA can effectively deal with multi-robot maps, whereas traditional BA methods fail to address them. The spatial BA is performed within a local spatial window around the loop, effectively leveraging multiple observations and abundant constraints from different robots. It directly optimizes the poses from different robots equally within the spatial window; thus, the poses of multiple robots matching the same geometric features (planes, lines) are adjusted simultaneously, \changed{reconstructing} a locally accurate and consistent map.
%BA局部准确性的扩展
Although spatial BA improves local map accuracy at loop closures, it lacks the capability for global map fusion. To address this problem, we propose a reasonable PGO-based method which effectively combines local BA constraints and odometry constraints to transfer the local alignment achieved by our spatial BA to the whole map, thereby achieving global consistency and accuracy.
%实验和表现
We conduct a series of extensive experiments, and the results demonstrate the high mapping accuracy and strong scalability of our multi-session map fusion approach, as shown in Fig.~\ref{fig:headimg}.  

%几点contribution
In general, the main contributions of this work can be summarized as follows: 
\begin{itemize}
    \item A scalable multi-session point cloud map merging and optimization system is designed, which integrates two-step PGO with spatial BA to achieve high-precision 3D mapping in large-scale and multi-robot scenarios. 
    \item \changed{A robust loop closure processing pipeline is designed, including outlier rejection and false negative loop recall, which enhances the reliability and completeness of loop constraints used for map fusion.}
    \item \changed{A novel spatial BA that can be used for multi-robot mapping is introduced, which operates on loop-based spatial windows to fully utilize loop constraints and jointly optimize multi-robot poses equally. This improves the local accuracy and reduces the serious map divergence.}
    \item \changed{A pose graph optimization scheme with reasonable sparsification of BA constraints is developed which effectively propagates local accuracy to the global scale, improving both consistency and accuracy.}
\end{itemize}

\section{Related Works}
\begin{figure*}[t]
  \centering
  \includegraphics[width=0.9\linewidth]{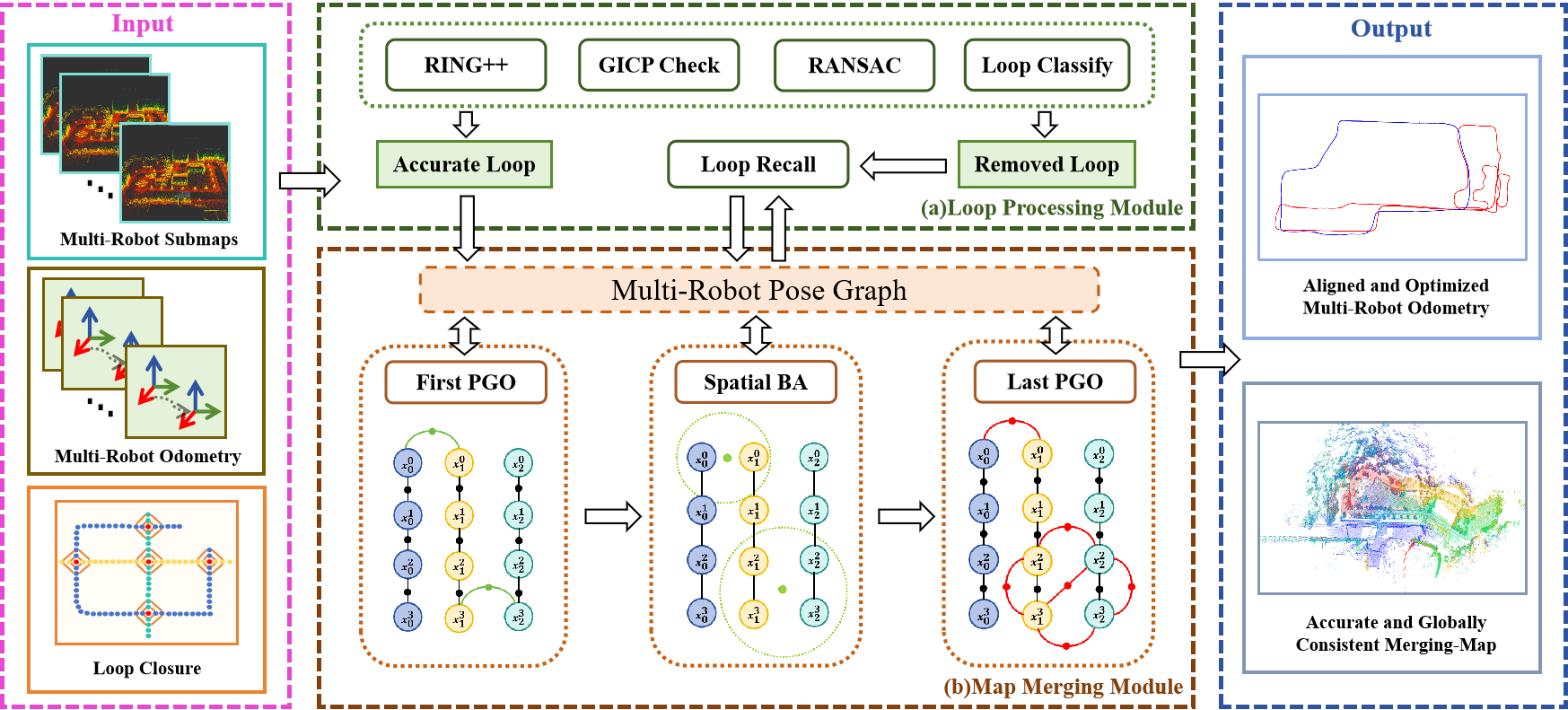}
  \caption{\textbf{The framework of our method}. It takes multi-robot submaps, odometry and loop closures as input, and \changed{reconstructs an} accurate and globally consistent merged map. (a) The Loop Processing Module, which removes outliers and recall false negative loops. (b) The Map Merging Module, which achieves multi-robot map merging through two-step PGO and spatial BA. Each of the three steps in Map Merging Module interacts closely with the multi-robot pose graph. 
  }
  \label{fig:framework}
\end{figure*}

\subsection{Single Map Maintenance and Optimization}

Single map maintenance and optimization have been extensively studied, with existing approaches falling broadly into two categories: PGO-based and BA-based methods. PGO remains a widely adopted back-end in LiDAR SLAM systems \cite{pgslam1, pgslam2, liosam, rised} by utilizing the constraints of loop closures to reduce the accumulated drift of odometry. However, traditional PGO frameworks focus on pose consistency rather than quality of the point cloud map , leading to maps with poor geometric quality.

In contrast, BA-based methods jointly optimize scan poses by minimizing the geometric residuals of matched primitives, offering improved map accuracy. BALM \cite{balm1} introduces a closed-form solution for feature parameters to reduce computational complexity, while BALM2 \cite{balm2} extends this with point clustering and a more efficient second-order solver. \changed{Recently, BALM3 \cite{BALM3.0} employs a majorization-minimization algorithm to decouple scan poses, reducing the time complexity to linear and enabling distributed optimization for large-scale mapping.} HBA \cite{hba} adopts a hierarchical BA strategy followed by top-down PGO, achieving scalable optimization in large-scale scenarios. PSS-BA \cite{pss-ba} introduces quadratic surface modeling for point cloud maps and uses progressive smoothing iterations to optimize map quality. RSO-BA \cite{rsoba} improves robustness via a second-order estimator integrated with a robust kernel function.
Despite their advantages, traditional BA approaches rely on time-based sliding windows and fail to incorporate long-range spatial constraints introduced by loop closures. Consequently, directly deploying these methods in multi-robot scenarios yields limited practicality. 

\subsection{Multi-Session Map Merging}

Multi-session map merging aims to integrate submaps from multiple agents either with or without initial pose estimates, into a unified and globally consistent map. 
% add LTA-OM
\changed{To facilitate this, several life-long mapping systems are proposed. LTA-OM \cite{lta-om} leverages a long-term association mechanism to seamlessly stitch live scans with pre-stored prior maps, eliminating the need for additional map-merging operations.}
% add slim
\changed{SLIM \cite{slim} tackles the scalability of multi-session mapping by parameterizing dense point clouds into structural lines and planes, combined with pose sparsification to drastically reduce memory footprints.}
% explore framework
Moreover, some multi-robot exploration frameworks such as SMMR-Explore \cite{smmr-explore} and MR-GMMExplore \cite{MR-GMMExplore} address multi-robot exploration and map fusion under communication limits. However, the former is limited to 2D point cloud submaps while the latter suffers from loss of geometric details caused by Gausssian Mixture Models (GMMS) based submaps. Therefore, these multi-robot exploration frameworks are unsuitable for \changed{reconstructing} large-scale, geometrically accurate point cloud maps.
SegMap \cite{segmap} extracts semantic features from 3D point clouds to estimate 6-DoF transformations and utilizes incremental PGO to achieve map fusion, but it heavily depends on accurate semantic segmentation. AutoMerge \cite{automerge} introduces a city-scale merging framework, but its performance degrades in complex environments due to the lack of elimination of loop outliers. LAMM \cite{LAMM} enhances place recognition by removing dynamic objects through M-Detector \cite{mdetector} and using the robust loop detection method BTC \cite{btc}, yet it only merges maps using PGO, which may suffer from severe local divergence.

Outlier rejection constitutes another critical challenge in multi-session mapping. RANSAC \cite{ransac} remains a standard solution for fitting models despite outliers, but its effectiveness diminishes under high outlier ratios and in cases lacking strong priors. \changed{PCM \cite{pcm1, pcm2, pcm3}, adopted in systems such as DCL-SLAM \cite{dclslam} and Disco-SLAM \cite{discoslam}, leverages pairwise geometric consistency for loop validation. Although robust to random outliers, its performance is compromised by accumulated odometry drift in multi-robot scenarios. Additionally, solving the NP-hard maximum clique problem leads to prohibitive computational times when dealing with dense loop candidates in large-scale environments.} GNC \cite{gnc} optimizes a sequence of graduated functions to achieve initialization-free global outlier rejection; however, its accuracy is still constrained by heterogeneous multi-robot odometry drifts.

% To address these limitations, we propose a two-stage loop closure filtering strategy. The first stage combines ICP convergence metrics with RANSAC to eliminate unreliable loop candidates, ensuring high precision among retained loops. The second stage employs a spatial-distance-based recall mechanism to recover false negatives by referencing the reliable loop set. This dual-stage process improves both the precision and recall of loop closure selection, providing a more robust foundation for accurate multi-session map merging.

\section{System Overview}
\label{subsection:system}
\subsection{Problem Formulation} 

Our objective is to \changed{reconstruct} an accurate and consistent map by merging multiple submaps from different robots. In LiDAR-based SLAM systems \cite{loam, dlio, fastlio2, slict}, \changed{each scanned point cloud is inherently registered to its corresponding LiDAR scan pose}, and the global point cloud map is constructed by \changed{registering} each scan into the world coordinate system. Therefore, the quality of the point cloud map is strictly associated with the sensor pose at the time of acquisition. 

In a multi-robot system, the set of pose sequences can be denoted as $S_N = \{s_1, s_2, \dots, s_N\}$, where each sequence $s_k$ corresponds to a separate robot. Different sequences such as $s_i$ and $s_j$, which have different starting points and initial orientations, lack prior information on relative pose transformation. Therefore, it is essential to estimate the relative transformations between these trajectories and optimize the poses using constraints in overlapping regions to improve geometric consistency.

The problem can be formulated as finding a set of optimized pose sequences $S_N^* = \{s_1^*, s_2^*,...,s_N^*\}$, such that all trajectories are aligned to a common coordinate system (typically $s_1^*$), and the \changed{subsequent} multi-robot map exhibits minimal divergence while maintaining global consistency.

\subsection{Framework of LEMON-Mapping}

Fig.~\ref{fig:framework} illustrates the overall framework of LEMON-Mapping. Our system consists of two main components: the Loop Processing Module detailed in Section \ref{subsection:loopprocessing} and the Map Merging Module described in Section~\ref{subsection:loopBA} and Section~\ref{subsection:mapmerging}. 

The Loop Processing Module ingests multi-agent odometry, submaps, and raw loop closure candidates (comprising both intra-robot and inter-robot loops), robustly filtering out erroneous constraints to output correct loop closures. The Map Merging Module then uses the loops to optimize the multi-robot trajectories through three steps: spatial bundle adjustment (BA) and two pose graph optimization (PGO) steps. 
% First, a preliminary global pose estimate is obtained using a conventional multi-robot PGO (FPGO), which serves as an initial guess and prior information for loop recall in the Loop Processing Module. Next, a spatial BA is performed within local regions around the loops to fully exploit the geometric constraints from overlapping observations. Finally, the last PGO (LPGO) step refines the trajectories using constraints of our spatial BA and constraints of odometry. 
As the two PGO steps share a similar formulation, they are jointly described in Section \ref{subsection:mapmerging}, while the proposed spatial BA is discussed separately in Section \ref{subsection:loopBA}.

\section{Loop Processing Module}
\label{subsection:loopprocessing}
The Loop Processing Module is a fundamental component that supports the spatial BA and the two-step pose graph optimization. It consists of three submodules: Outlier Rejection, Loop Classification, and Loop Recall. This module processes both self-loops and inter-loops from multiple robots, which are initially detected by the robust loop detection method RING++ \cite{ring++}.

\begin{figure}[!t]
  \centering
  \includegraphics[width=1.02\linewidth]{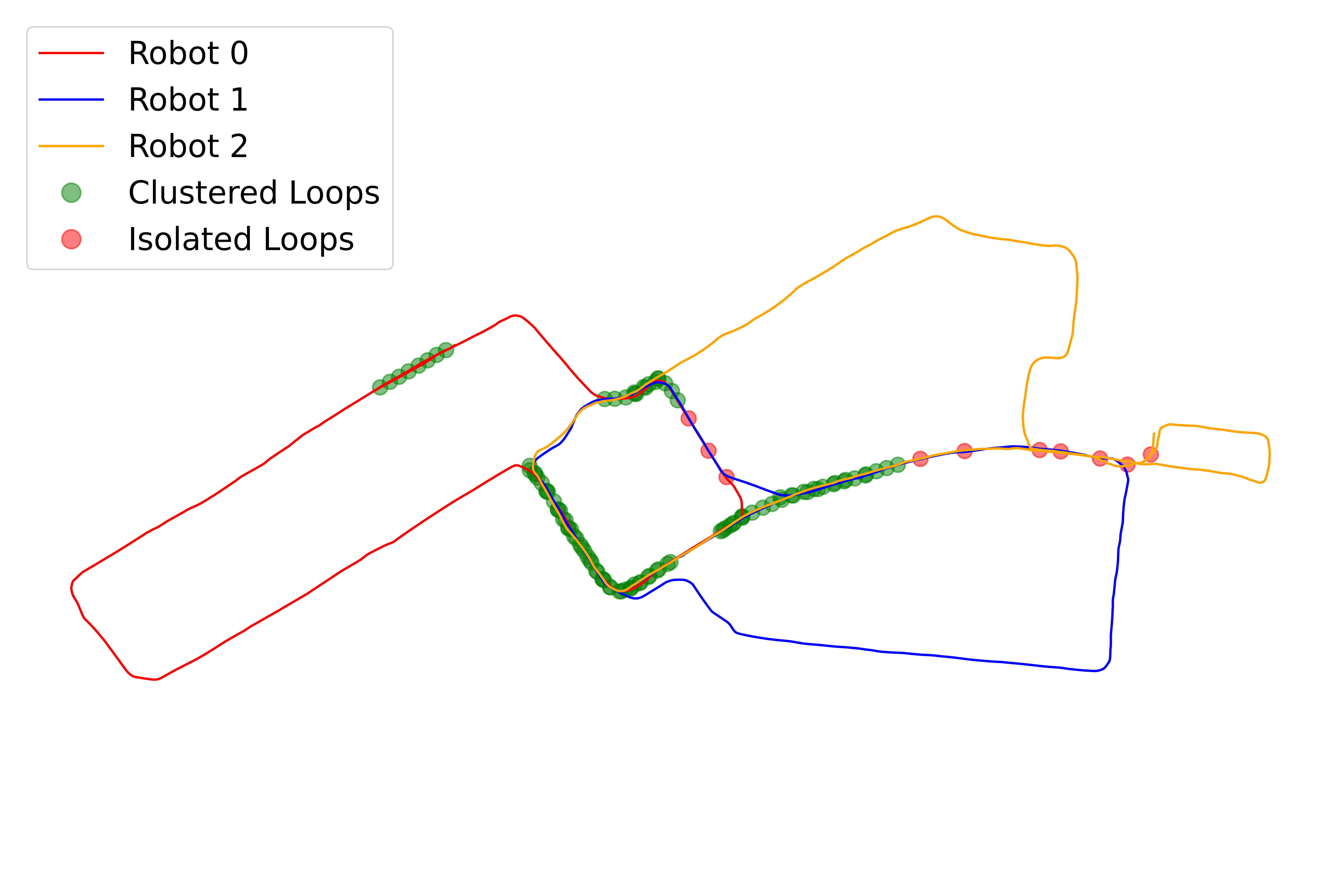}
  \vspace{-25pt}
  \caption{An example of loop classification in Campus\_1 dataset, the multi-robot trajectories and the two types of loops are shown as labeled.}
  \label{fig:loopclassify}
\end{figure}

%正确性检查
\subsection{Outlier Rejection}

Loop closure data may contain incorrect constraints, which seriously affects the accuracy of PGO. To solve this problem, we design an outlier rejection method using the Generalized Iterative Closest Point (GICP) \cite{gicp} combined with Random Sample Consensus (RANSAC) \cite{ransac}.

The process begins with Statistical Outlier Removal (SOR) \cite{sor}, which filters the noise from the raw point clouds by analyzing the statistical distribution of point-to-neighbor distances. GICP is then employed to align the filtered point clouds, initialized using the transformation estimated by RING++. To validate the alignment, correspondences between the two point clouds are refined using a RANSAC-based rejection method, which discards outliers inconsistent with rigid transformations. The loop is accepted only if the number of inlier correspondences exceeds a predefined threshold. Additionally, the GICP fitness score is used as a quantitative metric for alignment quality, and alignments with poor scores or non-convergence are discarded. This two-stage filtering approach enhances the reliability of loop closures by ensuring that only geometrically valid alignments are retained.

% 回环的分类
\subsection{Loop Classification} 

Following the outlier rejection process, the remaining valid loop closures are classified to facilitate efficient spatial BA. Specifically, we categorize loops into two types based on their spatial distribution: clustered loops and isolated loops.

We implement a region-growing algorithm based on Breadth-First Search (BFS). Starting from the spatial center of each loop, the algorithm incrementally searches for nearby loops within a predefined radius. If neighboring loops are found, the search expands from their centers, continuing recursively until no additional nearby loops are detected. Loops that form such spatial clusters are labeled as clustered loops, while those without adjacent loops are marked as isolated loops. An example of this classification process is shown in Fig.~\ref{fig:loopclassify}.

%回环重新召回
\subsection{Loop Recall} 

Since the outlier rejection step adopts strict criteria to ensure robustness and accuracy, it may inadvertently discard valid loop closures. These missing constraints between robot trajectories hinder the correction of divergence in overlapping areas, potentially degrading the consistency of the fused multi-robot map. To address this, we propose a loop recall mechanism to recover previously discarded valid loops. After the first PGO (Section~\ref{subsection:mapmerging}) aligns all robot trajectories into the same coordinate system using valid loops, the updated poses are used to re-evaluate previously rejected loops. If the Euclidean distance between the associated poses is below a threshold (2m in our system), the loop is recalled. This lightweight and distance-based strategy effectively recovers useful constraints and enhances subsequent optimization. 

\section{Spatial Bundle Adjustment}
\label{subsection:loopBA}
Traditional BA methods \cite{balm1,balm2,hba} jointly optimize sequential poses within a sliding window of temporally ordered data. However, they struggle with scenarios involving revisiting the same location over long time spans or from different agents, making them unsuitable for multi-robot systems. Different from them, our spatial BA jointly optimizes poses from different robots in local spatial windows simultaneously. This design reduces map divergence across sessions, as demonstrated in Fig.~\ref{fig:clusterBA} and Fig.~\ref{fig:compare_pcd}. Our BA specifically focuses on loop closure regions for two reasons: (1) these regions often exhibit significant spatial overlap among submaps and suffer from high divergence; (2) they contain multiple observations of the same geometric structure from different times or agents, offering rich constraints for accurate joint optimization.
 
%回环分类
The spatial BA is performed on all loops after the FPGO and the loop recall step. All available loops are divided into two categories based on the principle of loop classification in Section~\ref{subsection:loopprocessing}. The poses of all robots are built into a kd-tree using their spatial positions for faster searching. For each loop closure, we define a spherical spatial window centered at the midpoint between the two poses involved. We employ a radius-based search using the pose kd-tree to efficiently retrieve the poses in the spherical region. These selected poses form the local optimization window for the spatial BA, ensuring that only the spatially correlated data around the loop are refined. Depending on the type of loop closure classified in Section~\ref{subsection:loopprocessing}, we propose two forms of spatial BA designed for their respective characteristics.
For isolated loops, we apply our developed spatial diffusion bundle adjustment (DBA). For clustered loops, a spatial variant of hierarchical bundle adjustment (HBA) is proposed.

\begin{figure}[!t]
  \centering
  \includegraphics[width=1.0\linewidth]{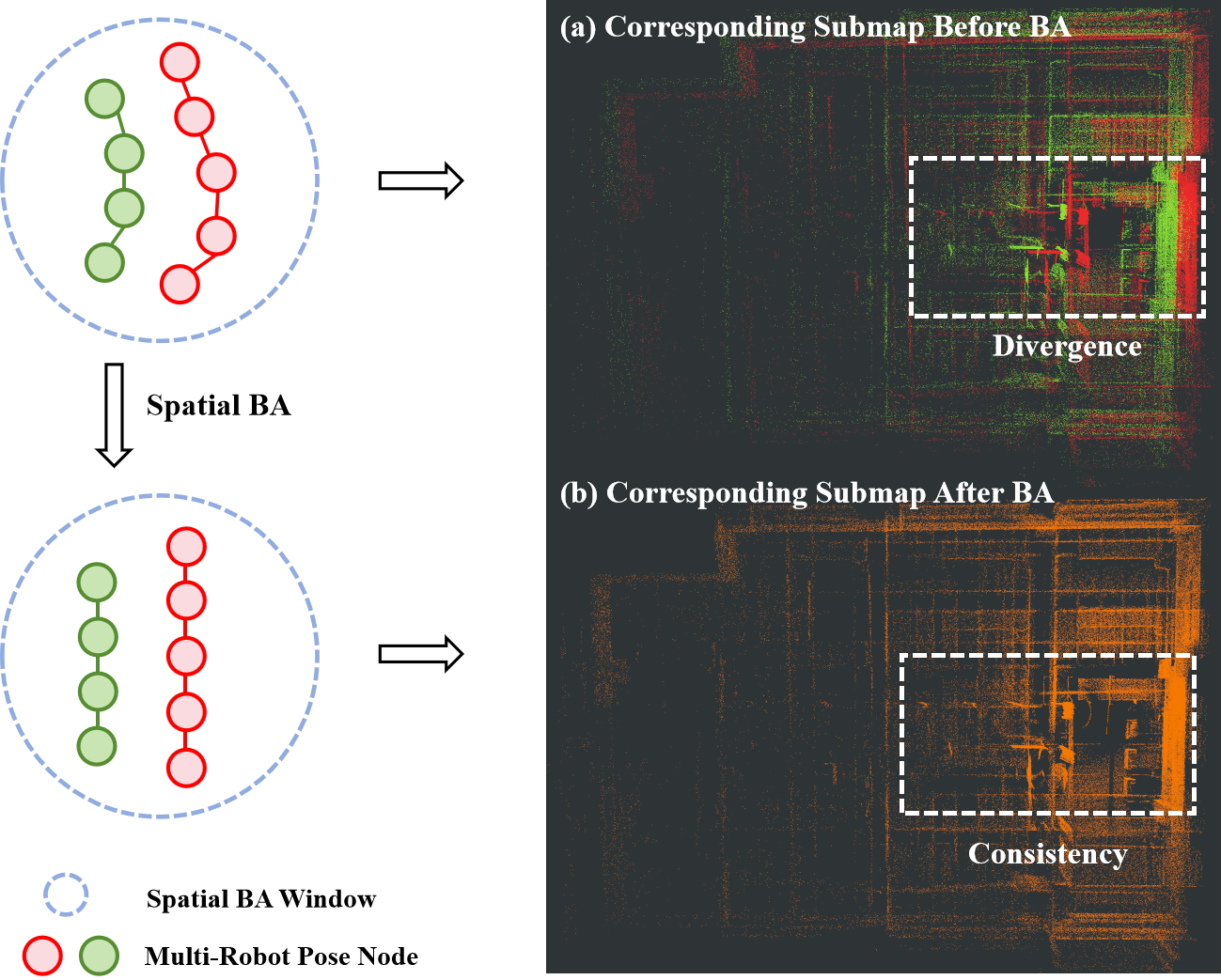}
  \caption{\changed{The red and green nodes in the left part show the trajectories of different robots at a certain loop. (a) and (b) show the submaps before and after BA optimization respectively. Our spatial BA significantly reduces divergence and reconstructs a locally consistent map.}}
  \label{fig:clusterBA}
\end{figure}

% 独立回环，扩散BA
\subsection{DBA for Isolated Loops} 

%对齐再进行BA
Our DBA is built upon BALM2 \cite{balm2}. However, it relies on accurate initial plane estimation, which may fail to converge reliably when the same plane observed by different robots appears misaligned due to map divergence. In such cases, the same structure may be voxelized as separate multi-layer planes (shown in Fig.~\ref{fig:clusterBA} (a)), making direct optimization \changed{difficult}. To address this, we partition the poses within the spatial window by robot into distinct clusters, each with its own point cloud, as illustrated in the left part of Fig.\ref{fig:clusterBA}. We then apply GICP \cite{gicp} to roughly align these clusters, reducing plane stratification and enabling reliable plane estimation for subsequent optimization.
%% add explanation
However, GICP only provides a coarse alignment of multi-robot submaps rather than true fine fusion. It simply adjusts the relative transformation of the assembled point cloud submaps of different pose sequence, but it does not optimize multi-robot poses at the frame level or refine point cloud at the single scan level.
After the initial pose graph optimization incorporating an isolated loop, the poses near the loop center are effectively aligned and optimized. However, as the distance from the loop center increases, the influence of loop constraints in FPGO gradually weakens, leading to less accurate optimization and alignment of distant poses.
Therefore, for an isolated loop, the key point is to figure out how large its influence range is and align the surrounding affected poses. We develop DBA to tackle this issue.

We begin with a set of LiDAR poses confined to a small range near an isolated loop. We then gradually expand the set by diffusely incorporating additional LiDAR poses from increasingly wider ranges. The key principle is that, within the loop influence region, poses located further from the loop center contribute less accurately to the representation of the loop's features. The details of DBA are as follows. 

For simplicity, we use the following notation, and we refer the reader to \cite{balm2} for more detailed information. In this paper, feature is seen as \textit{point clusters}, the point cluster for the $i$-th feature is denoted by set \( \bm{\mathcal{C_i}} = \{\mathbf{p}_{ijk} \in \mathbb{R}^3 | j=1,...,M_p, k=1,..., N_{ij}\}\), where $M_p$ is the number of poses, and the corresponding \textit{point cluster coordinate} $\boldsymbol{\Re}\!\left( \boldsymbol{\mathcal{C}} \right)$ is defined as:
\begin{equation}
\begin{gathered}
\boldsymbol{\Re}\!\left( \boldsymbol{\mathcal{C}_i} \right) \triangleq 
\sum_{j=1}^{M_p}\sum_{k=1}^{N_{ij}}
\begin{bmatrix}
    \mathbf p_{ijk} \\ 1
\end{bmatrix}
\begin{bmatrix}
    \mathbf p_{ijk}^T & 1
\end{bmatrix}
=
\begin{bmatrix}
    \mathbf P_i & \mathbf v_i \\
    \mathbf {v_i}^T & N_i
\end{bmatrix} \in \mathbb{S}^{4 \times 4}, \\
\mathbf P_i = \sum_{j=1}^{M_p}\sum_{k=1}^{N_{ij}} \mathbf p_{ijk} \mathbf p_{ijk}^T,
\quad 
\mathbf v_i = \sum_{j=1}^{M_p}\sum_{k=1}^{N_{ij}} \mathbf p_{ijk}.
\end{gathered}
\label{def-pc}
\end{equation}

A BA formulation with determining LiDAR poses $\mathbf T = (\mathbf T_1, \cdots, \mathbf T_{M_p})$, and feature parameters $\boldsymbol{\pi} = (\boldsymbol{\pi}_1, \cdots, \boldsymbol{\pi}_{M_f})$, where $M_f$ is the number of features, starts with the optimization problem:
\begin{align}
    \min_{\mathbf T, \boldsymbol{\pi}} 
	\Big(\sum\nolimits_{i=1}^{M_f} c (\boldsymbol{\pi}_i, \mathbf T)
	\Big). \label{BA-formulation}
\end{align}

\changed{When using plane features, each cost item denotes the squared point-to-plane Euclidean distance, which has been proven to take the form of (BALM2 \cite{balm2})}
\begin{equation}
\begin{gathered}
c_i(\mathbf T) \triangleq c(\pi_i, \mathbf T)  = \lambda_3 \left(\mathbf A \left(\sum_{j=1}^{M_p} \mathbf T_j \mathbf C_{f_{ij}} \mathbf T_j^T \right) \right), \\
\mathbf A (\mathbf C_i ) \triangleq \frac{1}{N_i} \mathbf P_i - \frac{1}{N_i^2} \mathbf v_i \mathbf v_i^T,\quad 
\mathbf C_i = \begin{bmatrix}
    \mathbf P_i & \mathbf v_i \\
    \mathbf v_i^T & N_i
\end{bmatrix} \in \mathbb{S}^{4 \times 4} ,
\end{gathered}
\label{cost_item}
\end{equation}
with $\lambda_3(\mathbf A)$ as the 3rd largest eigen value of matrix function $\mathbf A$, $\mathbf C_{f_{ij}} \in \mathbb{R}^{4 \times 4}$ being a pre-computed matrix where 
\begin{equation}
\begin{gathered}
\mathbf C_{f_{ij}} = 
\begin{bmatrix}
    \mathbf P_{f_{ij}} & \mathbf v_{f_{ij}} \\
    \mathbf v_{f_{ij}}^T & N_{ij}
\end{bmatrix}, \\
\mathbf P_{f_{ij}} = \sum_{k=1}^{N_{ij}} \mathbf p_{f_{ijk}} \mathbf p_{f_{ijk}}^T, \quad
\mathbf v_{f_{ij}} = \sum_{k=1}^{N_{ij}} \mathbf p_{f_{ijk}}.
\end{gathered}
\label{feature-pc}
\end{equation}

Leveraging the above definition, where \textit{point clusters} form plane features, we could solve the problem with an optimal update $\Delta \mathbf T^{\star}$ using the Levenberg-Marquardt (LM) algorithm:
\begin{align}\label{linear_equ}
    \Delta \mathbf T^{\star} = - \left( \mathbf H + \mu \mathbf I \right)^{-1} \mathbf J^T,
\end{align}

Here $\mu$ is the damping parameter, while $\mathbf{J}$ and $\mathbf{H}$ are the Jacobian and Hessian of the cost function, respectively.

In DBA, we categorize the poses into $D$ groups based on their distances from the loop. Each group contains a set of poses. The number of poses in the $d_i$-th group is denoted as $^{d_i}M_p$ (for $i=0,\dots,D$). In the following part, we first demonstrate that in this \textit{incremental} form of BA, the optimal update formulation \changed{provides a good approximation to} performing BA jointly on all features. We then prove that the algorithmic complexity is smaller compared to the traditional BA methods, and our approach achieves a relatively higher confidence level which should perform better in practice.

In the $d_i$-th diffusion process, we treat all poses that participated in the previous $i$ diffusion process as accurate and freeze their gradients, meaning they are not optimized any further. Consequently, when optimizing the $d_i$-th process, we can categorize the poses into two groups $^0\mathbf T$ and $^1\mathbf T$ with numbers of $M_{p_0} \triangleq \sum_{k=0}^{i-1} {^{d_k}M_p}$ and $M_{p_1} \triangleq {^{d_i}M_p} $.

\begin{assumption} \label{assumption:confidence}
    Small range LiDAR poses $M_{p_0}$ that participate in the BA have significantly smaller measurement noise covariance $\boldsymbol{\Sigma}_{c_{f_{ij}}}$ than that associated with the outer pose $M_{p_1}$. 
\end{assumption}

% \begin{lemma}[Continuity in Eigenvalue Functions] 
% \label{lemma:Sym2Continuos}
%     If $\mathbf A(x)$ is a continuous real symmetric matrix function (i.e., $\mathbf A(x)$ is continuous and is symmetric for every $x$), then its eigenvalue function $\lambda_i(x)$ is also continuous. 
% \end{lemma} % 谱定理可以很简单证明这个lemma

% \begin{proof}
% See Appendix (\ref{appendix:lemma1}).
% \end{proof}

According to the definition of matrix $\mathbf A$, it follows that $\mathbf A (\mathbf C ) \in \mathbb{S}^{3 \times 3}$. Hence, based on \changed{the differential assumption for cost function in~\cite{BALM3.0}}, the LM Jacobian and Hessian for DBA can be partitioned conformably as:

\[
\mathbf H = \begin{bmatrix} \mathbf H_{00} & \mathbf H_{01} \\[4pt] \mathbf H_{10} & \mathbf H_{11} \end{bmatrix}, \quad \mathbf{J} = \begin{bmatrix} \mathbf J_0 \\ \mathbf J_1 \end{bmatrix} .
\]

Using LM optimization in Eq. \ref{linear_equ}, the pose updates for group $M_{p_1}$ in the joint BA and DBA can be respectively derived as:
\begin{equation}
    \begin{aligned}
        & \Delta^1\mathbf{T}_{\text{joint}} = -\mathbf{S}^{-1} (\mathbf{J}_1 - \mathbf{H}_{10} (\mathbf{H}_{00} + \mu\mathbf{I}_0)^{-1} \mathbf{J}_0), \\
        &
        \Delta^1\mathbf{T}_{\text{DBA}} = (- (\mathbf{H}_{11} + \mu\mathbf{I}_1)^{-1} \mathbf{J}_1) ,
    \end{aligned}
    \label{eq:update}
\end{equation}
where $\mathbf{S} = \mathbf{H}_{11} + \mu\mathbf{I}_1 - \mathbf{H}_{10} (\mathbf{H}_{00} + \mu\mathbf{I}_0)^{-1} \mathbf{H}_{01}$ is the Schur's complement. 

The joint optimization can be viewed as an update method that refines the coupled terms $\mathbf H_{10}=\mathbf H_{01}^T$ for the DBA method. Therefore, from Eq. \ref{eq:update}, two refinement rates for the Hessian ($r_{\mathbf H}$) and Jacobian ($r_{\mathbf J}$) are defined as the ratio of the coupling contribution to the main term:
\begin{equation}
    \begin{aligned}
    & r_{\mathbf H}  =\frac{\| \mathbf{H}_{10} (\mathbf{H}_{00} + \mu\mathbf{I}_0)^{-1} \mathbf{H}_{01} \|}{\| \mathbf{H}_{11} + \mu\mathbf{I}_1 \|},\\
    & r_{\mathbf J} = \frac{\| \mathbf{H}_{10} (\mathbf{H}_{00} + \mu\mathbf{I}_0)^{-1} \mathbf{J}_0 \|}{\| \mathbf{J}_1 \|}.
    \end{aligned}
    \label{eq:update}
\end{equation}
% The smaller these two refinement rates are (i.e., closer to $0$), the less significant the coupled terms become relative to the main terms. According to Assumption \ref{assumption:confidence}, consequently, the information matrix (Hessian approximation) $\mathbf{H}_{00}$, which represents the certainty of the inner poses $M_{p_0}$, is assumed to be strong (i.e., possessing large singular values), formally implying that the norm of its pseudo-inverse is relatively small:
% $$\| (\mathbf{H}_{00} + \mu\mathbf{I}_0)^{-1} \| \approx 0.$$ 
% which in turn drives $r_{\mathbf{H}}$ and $r_{\mathbf{J}}$ towards zero. Consequently, the bias of the two methods is minimized, and the high-speed DBA method closely approximates the high-accuracy joint optimization of traditional BA.

\changed{The smaller these two refinement rates are, the less significant the omitted
coupling terms become relative to the active-block terms. Under
Assumption~\ref{assumption:confidence}, the inner poses are treated as
high-confidence anchors since they have been optimized in previous diffusion
steps and are supported by lower-noise measurements. In a local
noise-weighted least-squares interpretation, this corresponds to a stronger
inner information block \(\mathbf H_{00}\), making
\((\mathbf H_{00}+\mu\mathbf I_0)^{-1}\) relatively small when the inner block
is locally non-degenerate. As a result, the Schur correction terms
\(\mathbf H_{10}(\mathbf H_{00}+\mu\mathbf I_0)^{-1}\mathbf H_{01}\) and
\(\mathbf H_{10}(\mathbf H_{00}+\mu\mathbf I_0)^{-1}\mathbf J_0\) are suppressed.
Therefore, DBA can be interpreted as an active-block approximation of joint BA:
when \(r_{\mathbf H}\) and \(r_{\mathbf J}\) are small, the DBA update closely
matches the active-block update of the full joint optimization while avoiding
the cost of optimizing all poses jointly. In practice, this approximation is further guaranteed when selecting small cross-coupling between the frozen inner loop and active outer loop.}

\begin{lemma}[Computational Complexity Reduction of DBA]
\label{lemma:compute_DBA}
Let the total number of poses be \(M=\sum_{i=0}^{D-1} m_i\), where 
\(m_i \triangleq {^{d_i}M_p}\) denotes the number of poses in the \(i\)-th 
diffusion group, for simplicity. A joint BA optimization over all poses has complexity
\[
O\!\left(M_f M + M_f M^2 + M^3\right).
\]
In contrast, DBA performs BA incrementally and its total complexity is
\[
O\!\left(M_f\sum_i m_i \;+\; M_f\sum_i m_i^2 \;+\; \sum_i m_i^3 \right),
\]
which is strictly no larger and generally substantially smaller. In particular, 
when the diffusion groups are of comparable size, i.e., \(m_i \approx M/D\), the 
quadratic and cubic terms are reduced by factors of approximately \(D\) and 
\(D^2\), respectively.
\end{lemma}

\begin{proof}
    See Appendix (\ref{appendix:lemma2}). 
\end{proof}

This proof uses a maximum of $D$ diffusion steps, while in practice, DBA is used when having a strong confidence of inner group poses for a small number of times far below $D$.

Let us estimate the confidence level of the estimated pose using the covariance estimation. Denote $\mathbf C_{f} = \{ \mathbf C_{f_{ij}} \}$, $\delta \mathbf C_{f} = \{ \delta \mathbf C_{f_{ij}} \}$, and $\mathbf T^{\star}$ as the converged solution using the measured cluster $\mathbf C_{f}$, according to BALM2, we get:
\begin{align}
    \delta \mathbf T^{\star} &= \mathbf H^{-1}  \frac{\partial \mathbf J^T  \left(\mathbf T^{\star}, \mathbf C_{f} \right)}{\partial \mathbf C_{f}}   \delta \mathbf C_{f} \sim \mathcal{N} \left(\mathbf 0, \boldsymbol{\Sigma}_{\delta \mathbf T^{\star}} \right),  \label{noise-TtoC}\\
    \boldsymbol{\Sigma}_{\delta \mathbf T^{\star}} &= \mathbf H^{-1}  \frac{\partial \mathbf J^T  \left(\mathbf T^{\star}, \mathbf C_{f} \right)}{\partial \mathbf C_{f}}  \boldsymbol{\Sigma}_{\delta \mathbf C_f}  \frac{\mathbf J  \left(\mathbf T^{\star}, \mathbf C_{f} \right)}{\partial \mathbf C_{f}}  \mathbf H^{-T} \notag \\
    &= \mathbf H^{-1} \left(\sum_{i=1}^{M_f} \sum_{j=1}^{M_p} \mathbf L_{ij} \boldsymbol{\Sigma}_{c_{f_{ij}}} \mathbf L_{ij}^T\right) \mathbf H^{-T}.\label{pose_convariance}
\end{align}

And We can obtain 
\begin{align}
    \boldsymbol{\Sigma}_{c_{f_{ij}}} &= \sum_{k=1}^{N_{ij}} \mathbf B_{f_{ijk}} \boldsymbol{\Sigma}_{p_{f_{ijk}}} \mathbf B_{f_{ijk}}^T \succeq 0, \quad \mathbf B_{f_{ijk}} \in  \mathbb{R}^{9 \times 3}.\\
    \frac{\partial \mathbf J^T}{\partial c_{f_{ij}}} &= \begin{bmatrix}
	\vdots \\
        \frac{\partial \left( \mathbf J^p \right)^T}{\partial c_{f_{ij}}} \\
        \vdots
    \end{bmatrix} = \begin{bmatrix}
	\vdots \\
        \mathbf L^p_{ij} \\
        \vdots
    \end{bmatrix} \triangleq \mathbf L_{ij} \in \mathbb{R}^{6M_p \times 9} .
\end{align}

\begin{lemma}[Covariance ordering: DBA vs joint BA]
\label{lemma:covariance}
Similar to the partition above, let the Jacobian-derivative blocks w.r.t.\ feature clusters be
\[
\mathbf L_{ij} = \begin{bmatrix} {^0\mathbf L_{ij}} \\[3pt] {^1\mathbf L_{ij}} \end{bmatrix},
\qquad
\boldsymbol{\Sigma}_{c_{f_{ij}}}\succeq 0.
\]
Define the joint-pose covariance perturbation due to cluster noise as
\[
\boldsymbol{\Sigma}^{\mathrm{joint}}_{\delta\mathbf T^\star}
= \mathbf H^{-1}\Big(\sum_{i=1}^{M_f}\sum_{j=1}^{M_p}\mathbf L_{ij}\,\boldsymbol{\Sigma}_{c_{f_{ij}}}\,\mathbf L_{ij}^\top\Big)\mathbf H^{-T},
\]
and let $\boldsymbol{\Sigma}^{\mathrm{joint}}_{11}$ be its bottom-right block (covariance of block~1 under joint BA).  
For DBA (block~0 frozen, optimizing only block~1) define
\[
\boldsymbol{\Sigma}^{\mathrm{DBA}}_{1}
= \mathbf H_{11}^{-1}\Big(\sum_{i=1}^{M_f}\sum_{j\in\text{group1}} {^1\mathbf L_{ij}}\,\boldsymbol{\Sigma}_{c_{f_{ij}}}\,{^1\mathbf L_{ij}}^\top\Big)\mathbf H_{11}^{-T}.
\]
Then, under the usual invertibility assumptions on $\mathbf H$ and $\mathbf H_{00}$, the following PSD ordering holds:
\[
\boldsymbol{\Sigma}^{\mathrm{DBA}}_{1} \preceq \boldsymbol{\Sigma}^{\mathrm{joint}}_{11}.
\]
Consequently, freezing well-measured inner poses (block~0) yields a covariance for the active poses that is no larger (and typically smaller) than the covariance those same poses would have under a full joint BA.
Note that the inequality follows without assuming inner poses are better measured.
\end{lemma}

\begin{proof}
    See Appendix (\ref{appendix:lemma3}).
\end{proof}

% 密集回环，稀疏化后进行HBA
\subsection{HBA for Clustered Loops} 

\begin{figure}[!t]
  \centering
  \includegraphics[width=0.95\linewidth]{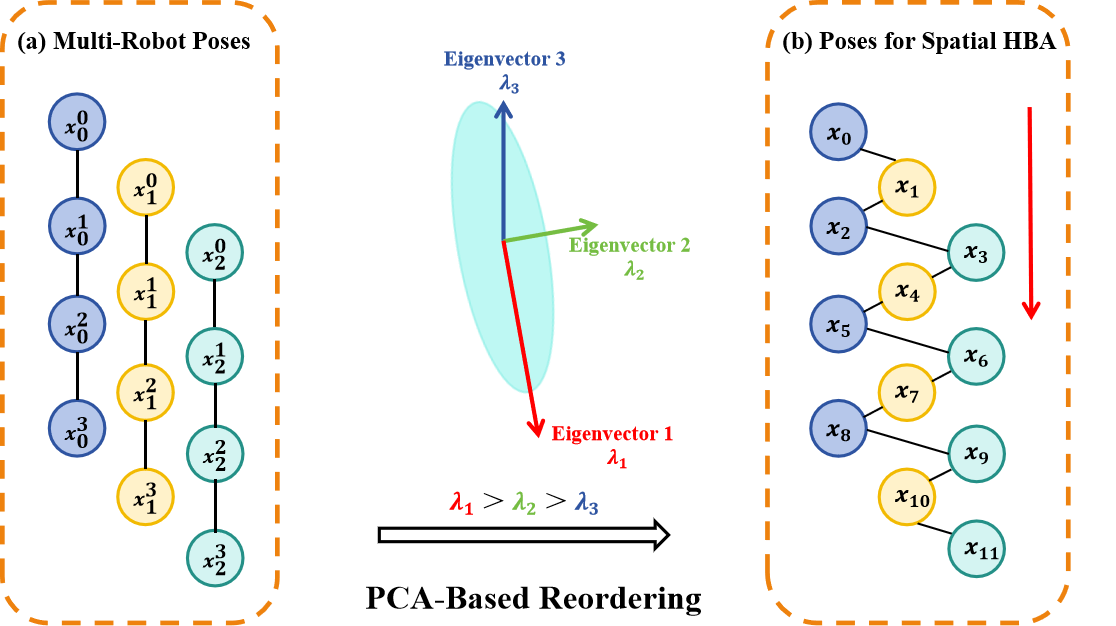}
  \caption{(a) shows the multi-robot poses in a spatial window. (b) shows the reordered poses using PCA for spatial HBA.}
  \label{fig:PCA}
\end{figure}

For two trajectories from different robots with significant spatial overlap, loop closures tend to be densely distributed, and the associated point cloud maps of adjacent loop closures often share large common areas. In such cases, it is crucial to jointly process the loops within the entire overlapping region to maintain local consistency. To address this, we adapt and extend HBA \cite{hba} into a multi-robot framework for clustered loop scenarios.

The original HBA relies on temporally ordered poses from a single robot, assuming that adjacent poses share common point cloud features. However, in our setting, poses selected via kd-tree radius search within a loop cluster are unordered and may come from different robots. To address this, we use Principal Component Analysis (PCA) \cite{pca} to analyze the spatial distribution of selected poses and reorder them along the principal axis defined by the largest eigenvalue, as illustrated in Fig.~\ref{fig:PCA}. This spatial reordering ensures that neighboring poses in the optimization sequence are also spatially close, promoting shared features and effective constraints. Once reordered, all poses are processed consistently in a unified optimization window, allowing the enhanced HBA to accurately align the geometry across overlapping regions from multiple robots.

\section{Two-Step Pose Graph Optimization}
\label{subsection:mapmerging}
The two-step PGO is used to align and refine the global structure of multi-robot trajectories. The details are as follows.

%first pgo
\begin{figure}[!t]
  \centering
  \includegraphics[width=1.0\linewidth]{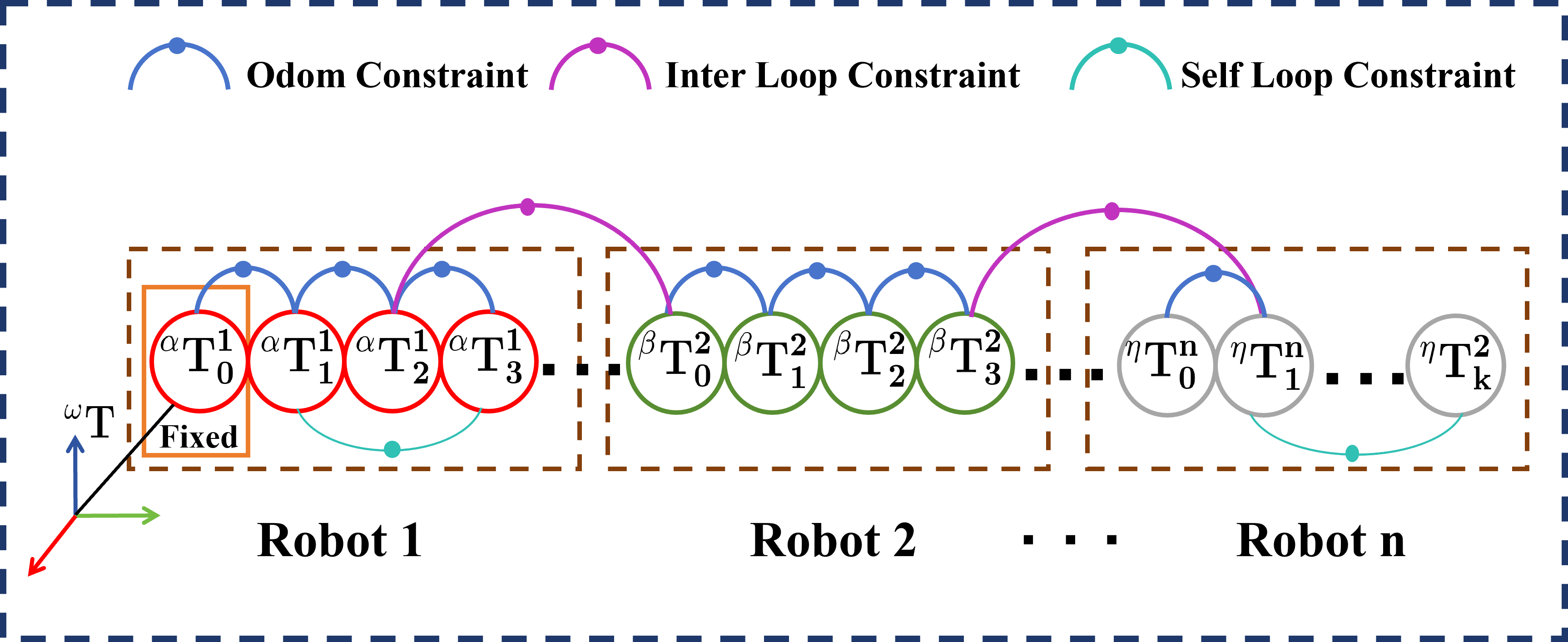}
  \caption{The first pose graph optimization, with constraints including odometry and loop closure. The first pose of the first robot is fixed as the world coordinate.}
  \label{fig:firstpgo}
\end{figure}

\subsection{First Pose Graph Optimization} 

Following loop processing, we perform the first pose graph optimization (FPGO) to coarsely align the trajectories of all robots. We construct a centralized pose graph shown in Fig.~\ref{fig:firstpgo}, in which each robot's pose sequence is included as a subgraph. The graph incorporates three types of constraint: odometry, self-loops, and inter-loops. The first pose of the first robot is fixed to define the origin point of the world coordinate, ensuring that other trajectories are aligned to this global reference frame. With this initial alignment established, we acquire a rough estimate of the relative positions among all robot trajectories, which enables us to recover the valid but previously rejected loops (in Section \ref{subsection:loopprocessing}). The recalled loops are added back to the set of loops and used as constraints in the FPGO. This
process containing optimization and recall is repeated until no new loops are recalled.

%最后一步pgo

\begin{figure}[!t]
  \centering
  \includegraphics[width=0.95\linewidth]{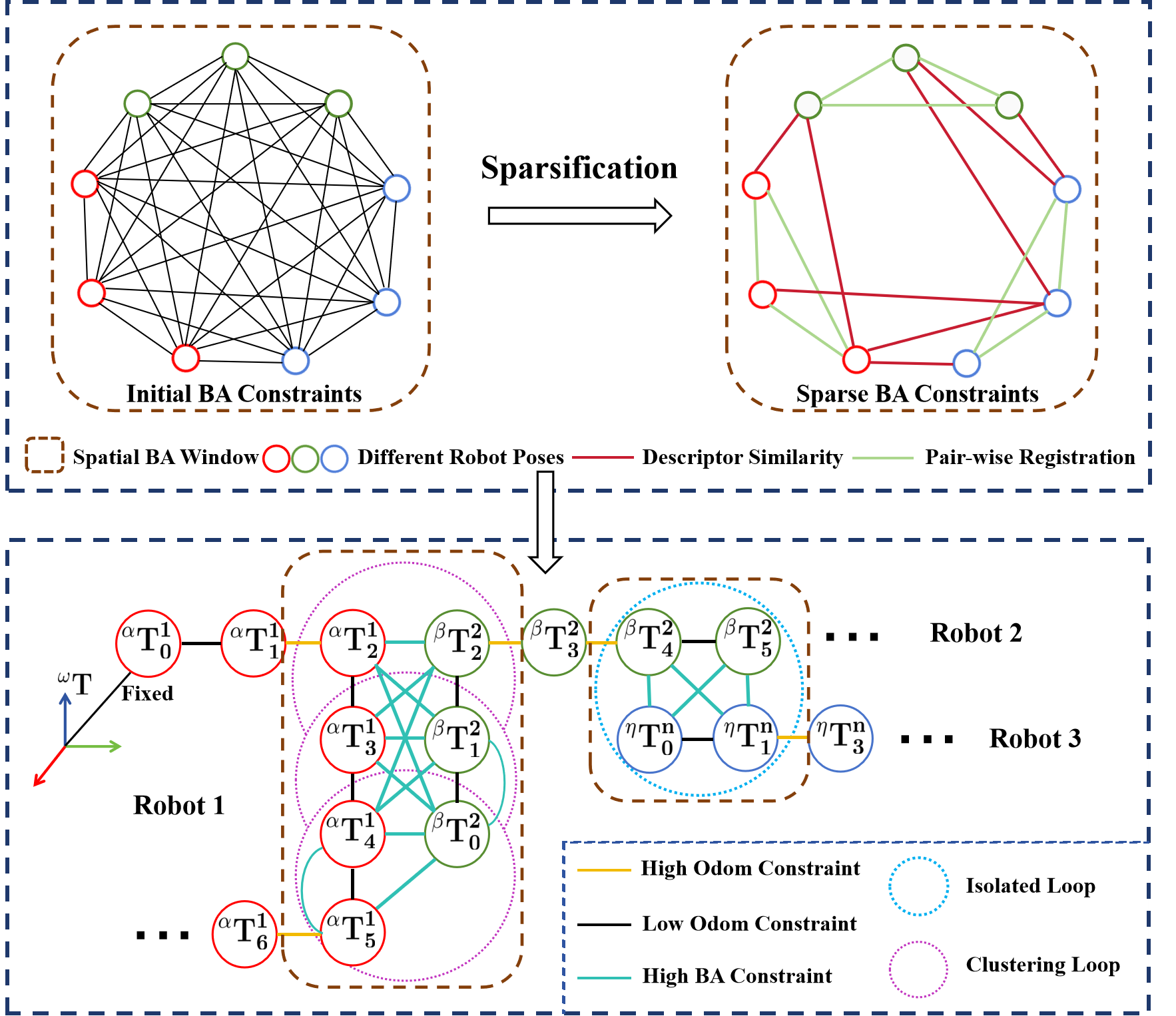}
  \caption{The upper part shows the sparsification of BA constraints in a spatial window at loop. The sparse BA constraints are add to pose graph below along with odometry constrains with different weights.}
  \label{fig:lastpgo}
\end{figure}

\subsection{Last Pose Graph Optimization}

%为什么做最后一步pgo
Although spatial BA improves the accuracy of local pose near loop closures, it does not achieve global consistency and may break the continuity of odometry. To address this issue, we develop the last pose graph optimization (LPGO).

We \changed{construct} the LPGO with two types of constraints: odometry constraints to maintain trajectory smoothness, and sparsified BA-based constraints to preserve refined local structures (see Fig.~\ref{fig:lastpgo}). For each robot, odometry constraints are added between adjacent poses to preserve temporal continuity. In cases where a BA-optimized pose is adjacent to an unoptimized one, a high-weight odometry constraint is applied to mitigate potential discontinuities introduced by local optimization. In contrast, lower weights are assigned between two unoptimized poses to allow greater global adjustment flexibility. Additionally, pose pairs within the spatial BA window that exhibit strong geometric overlap and reliable correspondence are assigned high-weight BA constraints to preserve local accuracy. 

We use two different metrics for constraint selection between pose pairs of different robots and pose pairs of the same robot. For different robots, we evaluate the similarity of their descriptors obtained from RING++ \cite{ring++} and add BA constraints only when the similarity exceeds a specified threshold. For pose pairs within the same robot, we apply pair-wise registration techniques from \cite{pss-goso} to selectively retain constraints between poses that significantly constrain each other. The sparsification process is as follows. 

For an edge $\mathbf{E}_k$ connecting two pose nodes belonging to the same robot, let $\mathbf{E}_{k}^{0}$ and $\mathbf{E}_{k}^{1}$ denote the indices of the corresponding poses. The residual $\boldsymbol{\epsilon}_k$ and its associated covariance matrix $\mathbf{\Omega}_k$ between the two poses can be estimated using a pair-wise registration approach. For clarity, we denote the poses at $\mathbf{E}_{k}^{0}$ and $\mathbf{E}_{k}^{1}$ as $(\mathbf{R}_{L_0}, \mathbf{t}_{L_0})$ and $(\mathbf{R}_{L_1}, \mathbf{t}_{L_1})$, respectively. The relative transformation from frame $L_0$ to $L_1$ is computed as $\mathbf{R}^{L_0}_{L_1} = \mathbf{R}_{L_0}^\top \mathbf{R}_{L_1}$ and $\mathbf{t}^{L_0}_{L_1} = \mathbf{R}_{L_0}^\top (\mathbf{t}_{L_1} - \mathbf{t}_{L_0})$. Based on nearest neighbor matching, we obtain point-to-point correspondences $\{(\mathbf{P}^{L_0}_u, \mathbf{P}^{L_1}_u)\}_{u=1}^{U}$, where $U$ denotes the total number of matched pairs. These correspondences are used to formulate the registration residual function $\boldsymbol{\epsilon}^{\mathrm{reg}}_k$ associated with edge $\mathbf{E}_k$.
\begin{equation}
  \begin{split}
  \boldsymbol{\epsilon}^{\mathrm{reg}}_k=\sum_{u=1}^{U}(\mathbf{R}^{L_0}_{L_1}\mathbf{P}^{L_1}_u+\mathbf{t}^{L_0}_{L_1}-\mathbf{P}^{L_0}_u),
  \end{split}
\end{equation}
The Jacobian of $\boldsymbol{\epsilon}^{\mathrm{reg}}_k$ respect to the two relative poses of the same robot $\mathbf{R}^{L_0}_{L_1},\mathbf{t}^{L_0}_{L_1}$ is calculated by \eqref{eq:jacobian_regis}.

\begin{equation}
  \begin{split}
  \mathbf{J}^{\mathrm{reg}}_k=\sum_{u=1}^{U}\left[ \begin{matrix}
    -\left[ \mathbf{P}_{u}^{L_0} \right] _{\times}&		\mathbf{0}\\
    \mathbf{0}&		\mathbf{I}
  \end{matrix} \right]. \label{eq:jacobian_regis}
  \end{split}
\end{equation}
The covariance $\mathbf{\Omega}_k$ of the registration function $\boldsymbol{\epsilon}^{\mathrm{reg}}_k$ can be calculated as follows.
\begin{equation}
  \begin{split}
  \mathbf{\Omega}_k=\mathbf{J}^{\mathrm{reg}\top}_k\mathbf{J}^{\mathrm{reg}}_k.
  \end{split}
\end{equation}
The minimal eigen value of the covariance $\lambda_k^{min} = \lambda_{min}(\Omega_k)$ can be used to represent the constraint ability between $\mathbf{E}_{k}^{0}$-th pose and $\mathbf{E}_{k}^1$-th pose. When the $\lambda_{min}(\Omega_k)$ of two node is small enough, the BA constraint between them will be retained and added to the pose graph in Fig. \ref{fig:lastpgo}. On the contrary, no constraints from BA will be added between $\mathbf{E}_{k}^{0}$-th pose and $\mathbf{E}_{k}^1$-th pose with a relatively large $\lambda_{min}(\Omega_k)$ value.

\section{Experiments}
\subsection{Experimental Setup}
All algorithms are implemented in C++ using the Robot Operating System (ROS) \cite{ros}. To evaluate the performance of LEMON-Mapping, we conduct experiments on several datasets, including the publicly available S3E \cite{s3e}, GEODE \cite{geode}, MARS-LVIG \cite{lvig}, R$^3$LIVE \cite{r3live} dataset, as well as a self-collected dataset. The GEODE, MARS-LVIG and R$^3$LIVE datasets are segmented into multiple sessions with overlap, but the starting point of each part is different and the relative transformation is unknown. 
The parameters of our self-collected dataset and the public datasets are shown in Table \ref{table:selfdataset} and Table \ref{table:opendataset}, respectively.
For all public datasets, initial odometry is obtained using the LiDAR-Inertial-Odometry method recommended by the respective dataset, while the initial odometry for our own dataset is generated using FAST-LIO2 \cite{fastlio2}.

The evaluation is structured as follows. We first compare our proposed spatial BA against two state-of-the-art baselines, BALM2 \cite{balm2} and HBA \cite{hba}, in Section \ref{subsection:singlerobotstudy}. To assess full-system localization performance, we benchmark our framework against DCL-SLAM \cite{dclslam} and LAMM \cite{LAMM} in Section \ref{subsection:comparestudy}. Mapping quality and loop closure processing are then thoroughly evaluated in Sections \ref{subsection:mapping quality} and \ref{subsection:loopstudy}, respectively. Furthermore, an ablation study is conducted in Section \ref{subsection:ablationstudy} to dissect the individual contributions of key components. We also investigate the scalability of our approach under an increasing number of sessions (Section \ref{subsection:scalabilitystdudy}). Finally, Section \ref{subsec:efficiency_analysis} analyzes runtime and memory efficiency to demonstrate the practical viability of our system.

\begin{table}[t]
  \centering
  \caption{Overview of Self-Collect Dataset}
  \label{table:selfdataset}
  \begin{threeparttable}
  \begin{tabular*}{0.48\textwidth}{@{\hspace{5pt}} @{\extracolsep{\fill}} c c c c @{\hspace{5pt}}}
    \toprule
    \makecell{Sequence} & \makecell{Environment} & \makecell{Trajectory Length (m)} & \makecell{LiDAR}\\
    \midrule
    Garage & Indoor & 334 & Mid360 \\
    Library & Outdoor & 519 & Mid360 \\
    Yard & Indoor \& Outdoor & 232 & Mid360 \\
    Laboratory &  Outdoor & 98 & Mid360 \\
    Flying Arena & Indoor & 280 & Mid360 \\
    \bottomrule
  \end{tabular*}
  \end{threeparttable}
\end{table}

\begin{table}[t]
  \centering
  \caption{Overview of Public Dataset}
  \label{table:opendataset}
  \begin{threeparttable}
  \begin{tabular*}{0.48\textwidth}{@{  } @{\extracolsep{\fill}} c c c c c @{  }}
    \toprule
    \makecell{Dataset} & 
    \makecell{Sequence} &
    \makecell{Robot \\ Number} &
    \makecell{Trajectory \\ Length (m)} & 
    \makecell{LiDAR} \\
    \midrule
    \multirow{6}{*}{S3E}
        & Campus\_1  & 3 & 2989 & Velodyne \\
        & Campus\_3 & 3 & 2938 & Velodyne \\
        & Dormitory & 3 & 2168 & Velodyne \\
        & Library & 3 & 1524 & Velodyne \\
        & Teaching\_Building & 3 & 1983 & Velodyne \\
        & Tunnel & 3 & 1525 & Velodyne  \\
        \midrule
        \multirow{4}{*}{GEODE}
        & Inlandwaterways & 3 & 1209 & Avia\\
        & Tunnelingtunnel & 3 & 270 & Avia  \\
        & Stairs & 2 & 161 & Ouster \\
        & Offroad & 2 & 542 & Ouster \\
    \midrule
    \multirow{4}{*}{MARS-LVIG}
        & Airport & 2 & 4135 & Avia\\
        & Valley & 3 & 6697 & Avia \\
        & Town & 4 & 6117 & Avia \\
        & Island & 2 & 2139 & Avia \\
    \midrule
    \multirow{3}{*}{R$^3$LIVE}
        & HKU Park & 5 & 356 & Avia \\
        & HKU Campus & 10 & 338 & Avia \\
        & HKUST Campus & 20 & 945 & Avia \\
        \bottomrule
  \end{tabular*}
  \end{threeparttable}
\end{table}

\begin{figure}[!t]
  \centering
  \includegraphics[width=0.96\linewidth]{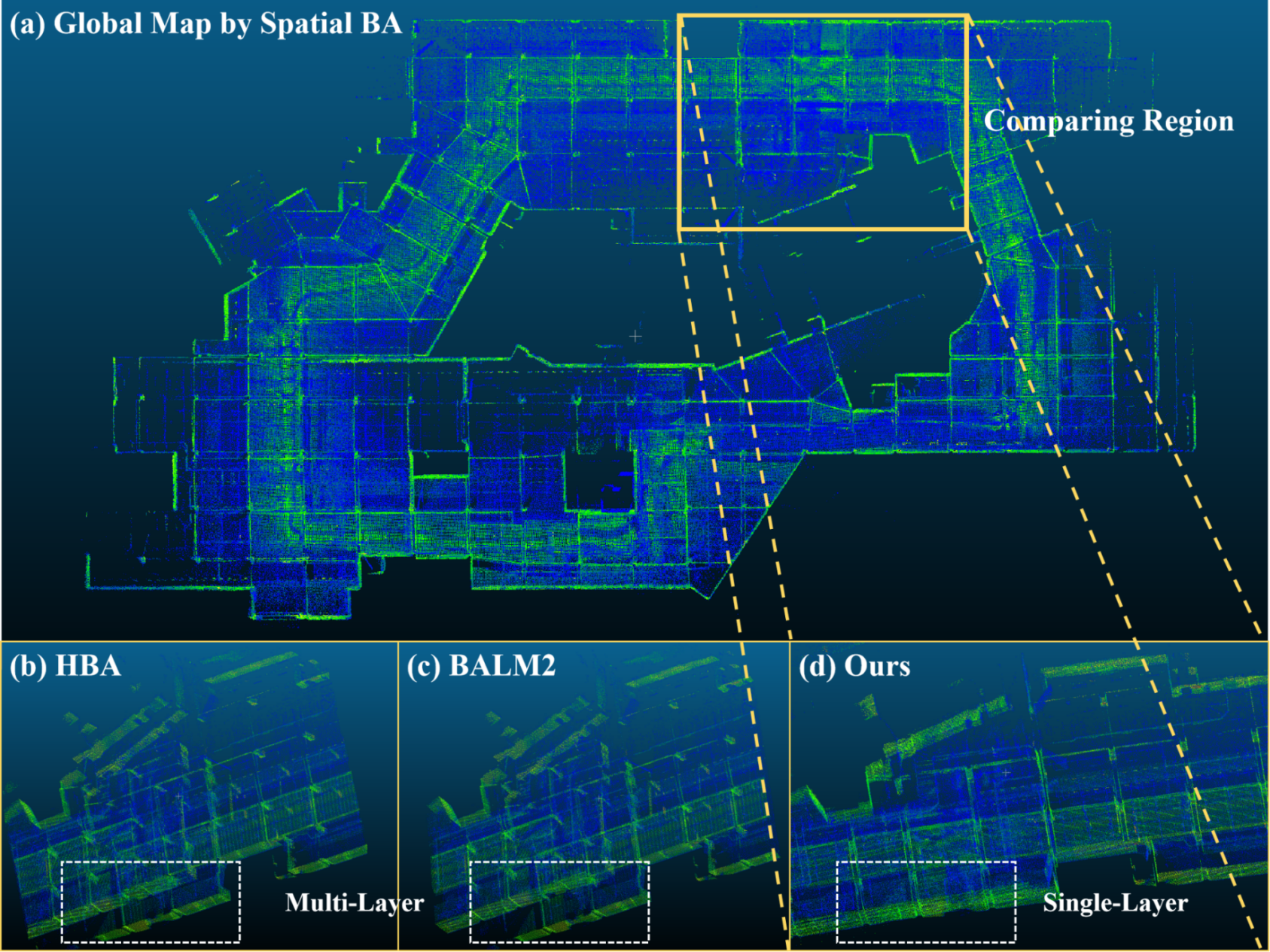}
  \caption{Result of single-robot study in our Garage dataset. (a) shows the global map generated by our method. (b)-(d) show the details of the local maps near the loop closure for different methods.}
  \label{fig:single}
  
\end{figure}
%%%%%%%%%%%%%%%%%%%%%%%%%%%%%%%%%%%%%%%%%%%%%%%%%%%%%%%%%%%%%%%%%%%%%%%%%%%%%%%%%%%%单机实验%%%%%%%%%%%%%%%%%%%%%%%%%%%%%%%

%单机实验表
\begin{table}[t]
\centering
\caption{MME, $z$-Drift AND $z$-RMSE of Single-Robot Study}
\label{table:singlerobot}
\begin{threeparttable}
\begin{tabular*}{0.48\textwidth}{@{\hspace{10pt}} @{\extracolsep{\fill}} c c c c c @{\hspace{10pt}}}
\toprule
\makecell{Sequence} & \makecell{Method} & \makecell{MME} & \makecell{$z$-Drift} & \makecell{$z$-RMSE} \\
\midrule
\multirow{3}{*}{Garage} 
    & HBA    & -6.83 & 5.98 & 7.14 \\
    & BALM2  & \textbf{-6.93} & 5.95 & 7.17 \\
    & Ours   & -6.86 & \textbf{4.87} & \textbf{5.35} \\
\midrule
\multirow{3}{*}{Library} 
    & HBA    & -5.95 & 6.17 & 6.78 \\
    & BALM2  & -6.20 & 6.51 & 7.31 \\
    & Ours   & \textbf{-6.23} & \textbf{3.68} & \textbf{4.53} \\
\midrule
\multirow{3}{*}{Yard} 
    & HBA    & -6.39 & 2.33 & 2.70 \\
    & BALM2  & -5.93 & 1.86 & 2.12 \\
    & Ours   & \textbf{-6.40} & \textbf{1.12} & \textbf{1.25} \\
\midrule
\multirow{3}{*}{Laboratory} 
    & HBA    & \textbf{-6.26} & 3.24 & 3.72 \\
    & BALM2  & -6.04 & 3.22 & 3.70 \\
    & Ours   & -6.20 & \textbf{3.17} & \textbf{3.64} \\
\bottomrule
\end{tabular*}
\end{threeparttable}
\end{table}

\begin{figure*}[!t]
  \centering
  \includegraphics[width=0.96\linewidth]{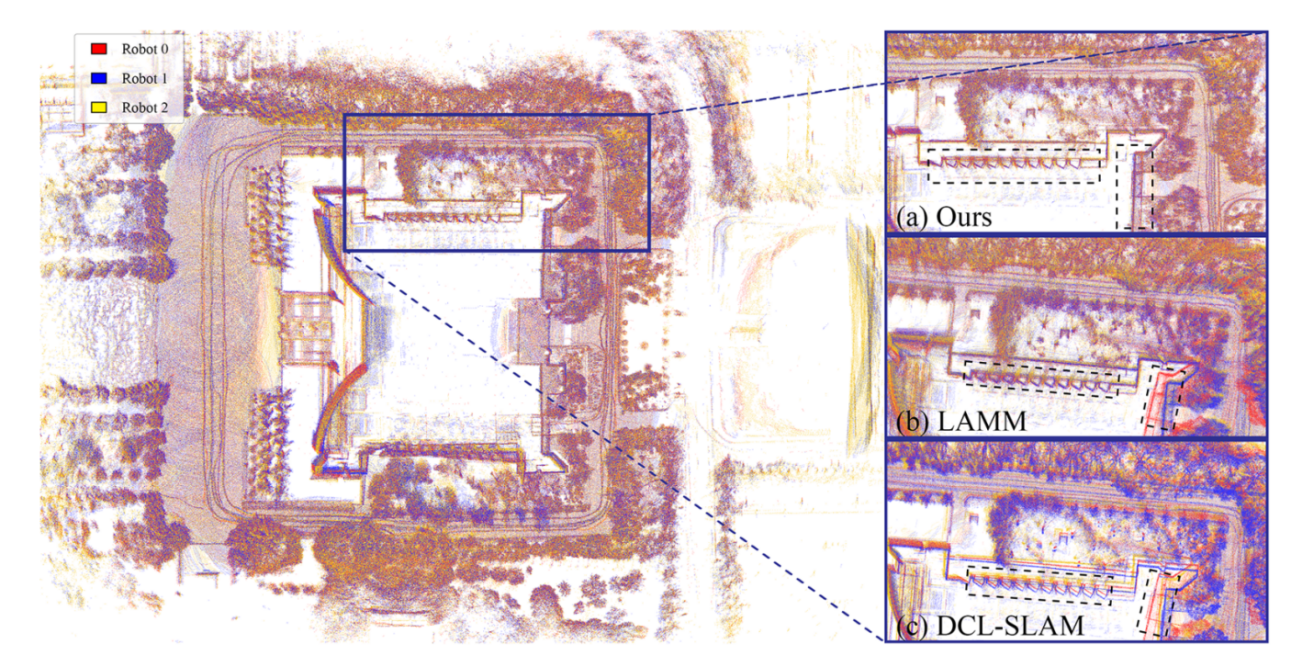}
  \caption{The map merging result of S3E library. The local maps \changed{reconstructed} by three methods are selected for comparison. Our method almost eliminates the divergence of submaps. While LAMM \cite{LAMM} and DCL-SLAM \cite{dclslam} suffer from serious inconsistencies.} 
  \label{fig:compare_pcd}
\end{figure*}

\subsection{Single-Robot Study}
\label{subsection:singlerobotstudy}

We utilize our self-collected single-robot dataset consisting of several indoor and outdoor scenes (Table \ref{table:selfdataset}) here. During data acquisition, the $z$-axis value remains largely stable across each scene, allowing it to serve as a reference to evaluate vertical drift.

Each sequence in our dataset includes loop closures that enable spatial BA. We compare our spatial BA with BALM2 \cite{balm2} and HBA \cite{hba}, using average $z$-axis drift ($z$-DRIFT) and $z$-axis Root Mean Square Error ($z$-RMSE) as primary evaluation metrics. Both metrics are computed with respect to the $z$-value of the first frame as a reference, allowing a consistent evaluation of vertical alignment over time. Due to the absence of ground truth, we also compute the Mean Map Entropy (MME) using MapEval \cite{mapeval}, where lower values indicate better map consistency and reduced clutter. To ensure a fair comparison of the performance of the three BA methods, all methods are evaluated using raw odometry trajectories without any prior loop-based refinement. For computational feasibility in large-scale scenes, BALM2 is executed in a sliding-window configuration.

Table \ref{table:singlerobot} presents the results. Our method achieves the best performance across all metrics in the Library and Yard scenarios. Furthermore, it consistently outperforms BALM2 and HBA in $z$-DRIFT and $z$-RMSE, demonstrating its superior capability in mitigating global drift. Fig.~\ref{fig:single} shows the mapping performance of HBA, BALM2 and our proposed method in the Garage sequence, with the framed area indicating the region where loop closures occur. It can be seen that baseline methods exhibit significant layering in this area, whereas our method achieves superior alignment. This result clearly demonstrates the strong capability of our spatial BA approach ins handling loop closure regions.

\begin{figure*}[!t]
  \centering
  \vspace{-0.5cm}
  \includegraphics[width=0.96\linewidth]{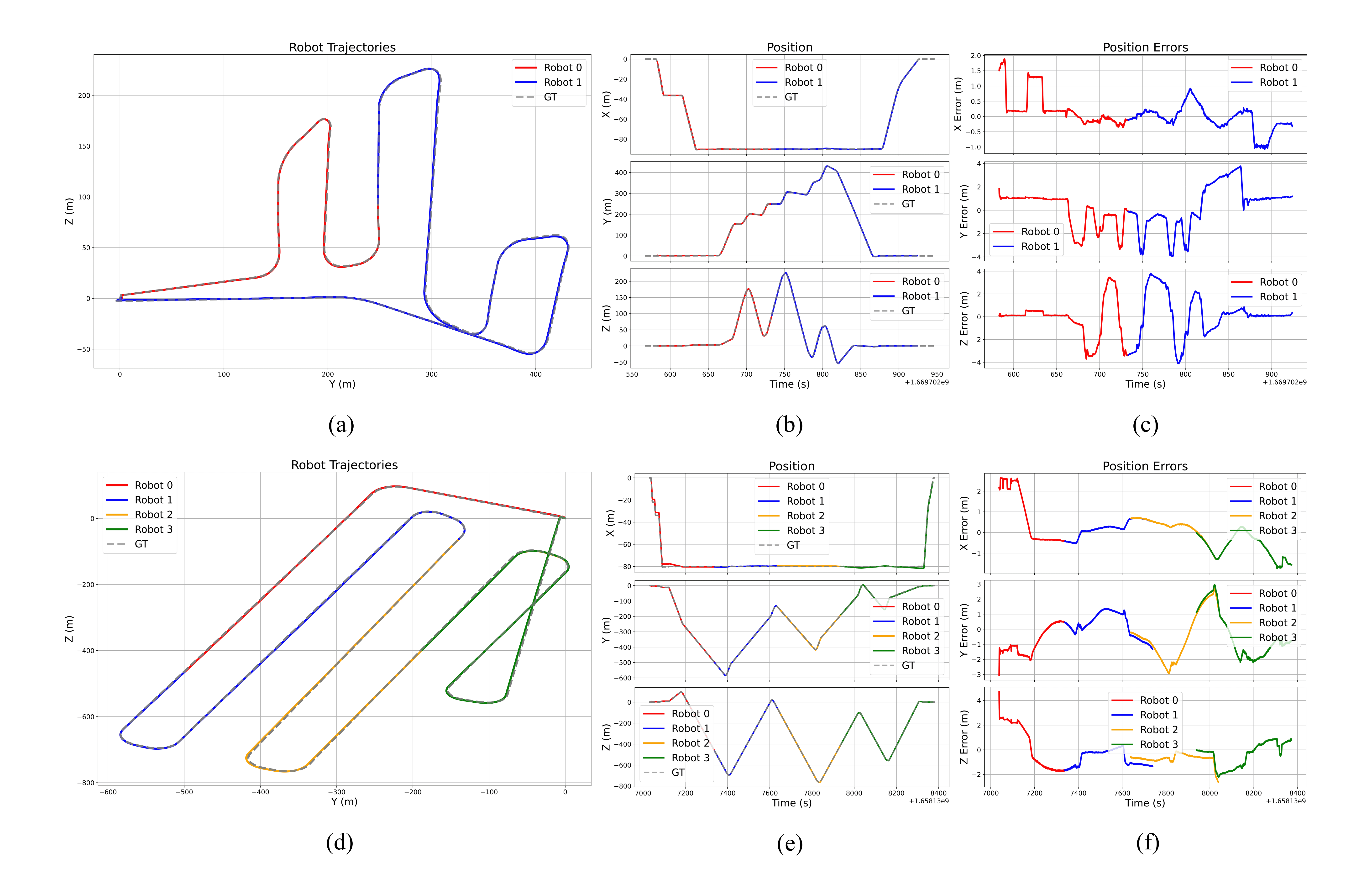}
  \vspace{-0.5cm}
  \caption{The optimized trajectory, position, and position error for proposed method in MARS-LVIG Island ((a)-(c)) and Town ((d)-(f)). 
  }
  \label{fig:evaluate}
\end{figure*}

\begin{figure}[!t]
  \centering
  \includegraphics[width=0.96\linewidth]{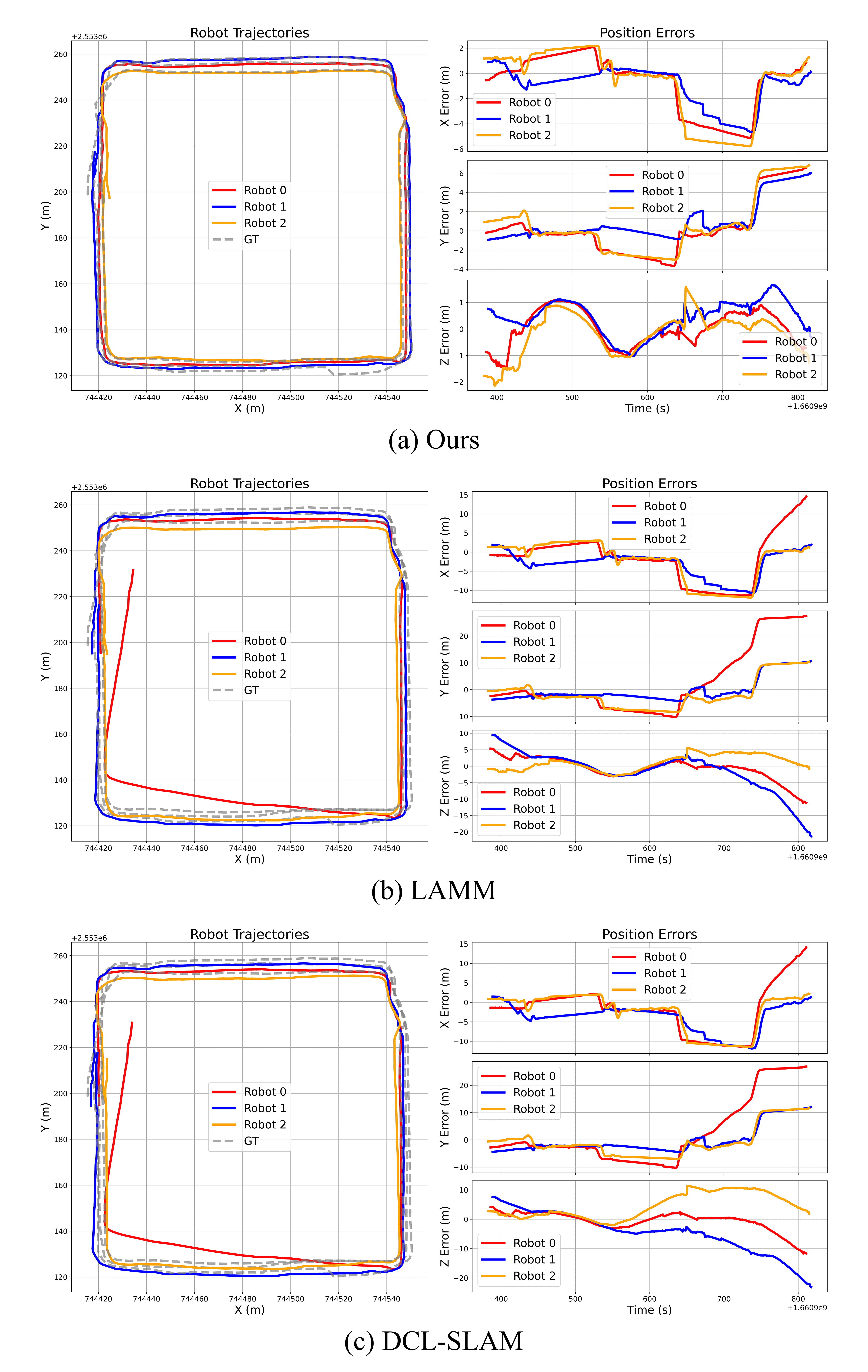}
  \caption{Comparison of multi-robot trajectories and error between our method, LAMM, and DCL-SLAM on S3E Library. The ground truth is shown as a dashed line in the figure.}
  \label{fig:compare_library}
\end{figure}

\begin{figure}[!t]
  \centering
  \includegraphics[width=0.96\linewidth]{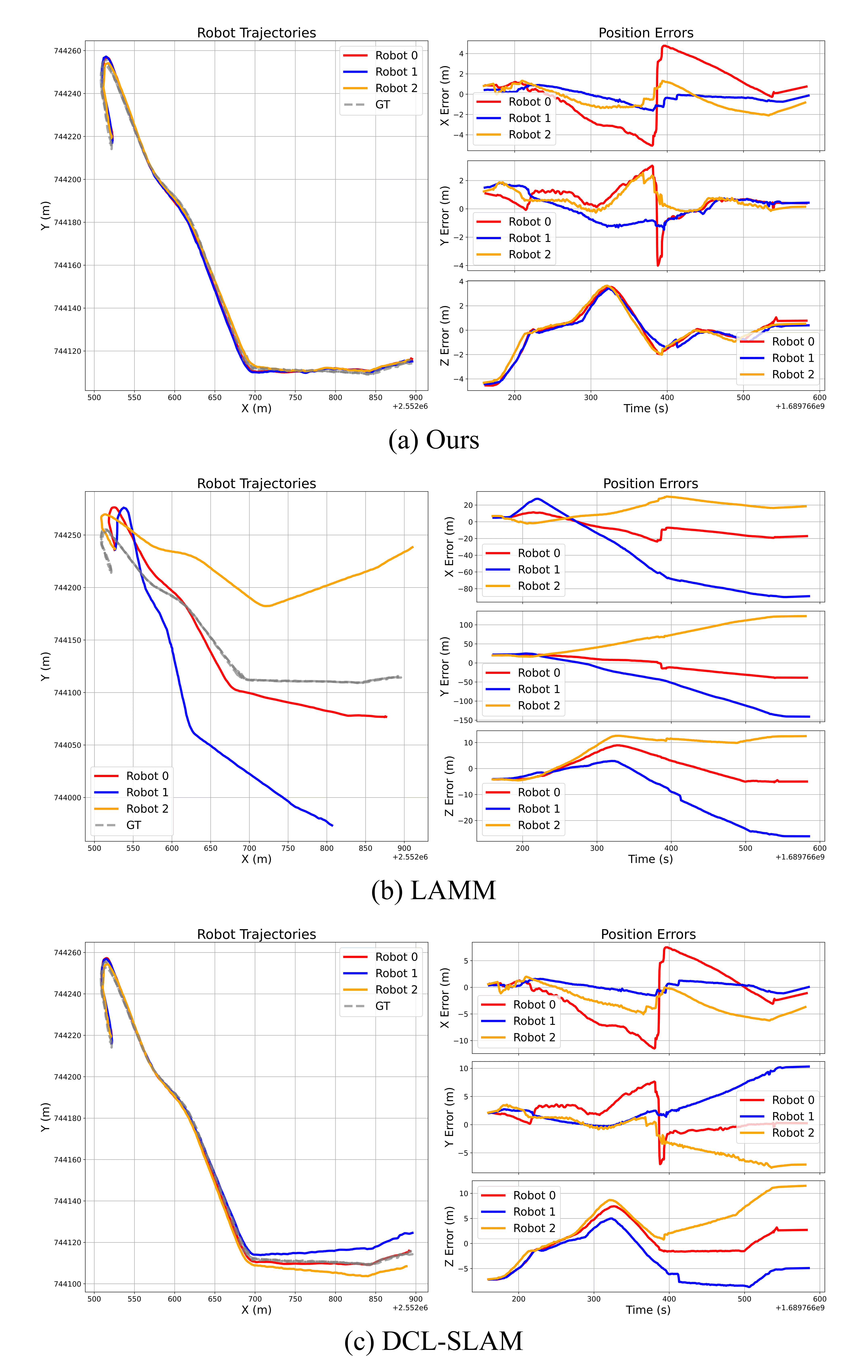}
  \caption{The same comparison of multi-robot trajectories on S3E Tunnel as Fig.~\ref{fig:compare_library}.}
  \label{fig:compare_tunnel}
\end{figure}
%%%%%%%%%%%%%%%%%%%%%%%%%%%%%%%%%%%%%%%%%%%%%%%%%%%%%%%%%%%%%%%%%%%%%%%%%%%%%%%%%%%%多机实验%%%%%%%%%%%%%%%%%%%%%%%%%%%%%%%
\subsection{Multi-Robot Localization Study}
\label{subsection:comparestudy}

%多机实验表
\begin{table}[t]
\centering
\caption{RMSE of The ATE(m) of Localization Study}
\label{table:compare}
    \begin{threeparttable}
        \begin{tabular*}{0.48\textwidth}{@{\hspace{15pt}} @{\extracolsep{\fill}} c c c c @{\hspace{15pt}}}  
        \toprule
        \makecell{Sequence} & \makecell{DCL-SLAM} & \makecell{LAMM} & \makecell{Ours}\\
        \midrule
        Campus\_3 & 17.89 & 12.51 & \textbf{3.51} \\
        Dormitory & 3.52 & 28.67 & \textbf{3.46} \\
        Library & 6.37 & 4.58 & \textbf{1.48} \\
        Tunnel & 3.09 & \texttimes & \textbf{0.98} \\
        Inlandwaterways & \texttimes & \texttimes & \textbf{7.44} \\
        Offroad & \texttimes & 25.45 & \textbf{1.30} \\
        Tunnelingtunnel & \texttimes  & 0.62 & \textbf{0.16} \\ 
        Airport & \texttimes  & 8.41 & \textbf{1.30} \\ 
        Town & \texttimes & \texttimes & \textbf{1.49} \\
        Island & \texttimes & \texttimes & \textbf{0.801} \\
        \bottomrule
        \end{tabular*}
    \end{threeparttable}
\end{table}

To further evaluate the proposed framework, we perform comparative multi-robot experiments using the MARS-LVIG \cite{lvig}, GEODE \cite{geode}, and multi-robot S3E \cite{s3e} datasets. Our framework is compared with the multi-robot SLAM system DCL-SLAM \cite{dclslam} and the multi-session map merging method LAMM \cite{LAMM}.

The accuracy of map merging is evaluated using the Root Mean Square Error (RMSE) of the Absolute Trajectory Error (ATE) in meters. A failure is defined as any sequence with an RMSE greater than 30 meters and is indicated by “\texttimes” in the table. As shown in Table \ref{table:compare}, the best values are highlighted in bold. Our framework successfully merges all sequences, while LAMM and DCL-SLAM fail in four and six sequences, respectively. In successful cases, our framework achieves significantly lower RMSE values, indicating higher merging accuracy.

We select a scene where all three methods successfully merge the multi-robot point cloud for display. Fig.~\ref{fig:compare_pcd} visualizes the results from the S3E Library sequence. The left image shows our globally consistent and accurate merged map, while the right part shows the local fused map of the three methods. Our method yields tightly aligned local maps, whereas LAMM and DCL-SLAM exhibit severe local divergences. 

The comparison of trajectory and ground truth further confirms the superior accuracy of our method. 
Fig.~\ref{fig:evaluate} shows the optimized trajectories and error plots for different scenes. The multi-robot trajectories generated by our framework are highly consistent with the ground truth, while the tracking errors remain low and stable, thereby demonstrating the superior performance and robustness of our method. Fig.~\ref{fig:compare_library} and Fig.~\ref{fig:compare_tunnel} presents a comparison of multi-robot trajectories between our method and the baselines. The trajectories estimated by our method closely follow the ground truth for each robot, exhibiting low and stable errors. In contrast, the benchmark methods show significant deviations in some regions and minor local inconsistencies in others. The large deviations are likely caused by incorrect loop closures, while the local inconsistencies may result from the lack of bundle adjustment, which fails to eliminate the divergence of submaps.

The performance disparity can be attributed to LAMM and DCL-SLAM relying solely on PGO with loop closure constraints, which neglects the refinement of overlapping local regions. Lack of attention to the geometric structure of point cloud maps leads to serious local multi-robot map divergence. Our framework explicitly re-examines the role of loop closures and enhances utilization by performing spatial BA in local regions, while propagating the refined results globally via the last pose graph optimization. This process effectively reduces divergence and ensures consistent multi-session map fusion. 

\subsection{Multi-Robot Mapping Quality Evaluation}
\label{subsection:mapping quality}

In this section, we evaluate the multi-robot mapping quality. For the MARS-LVIG dataset, the ground-truth map is a high-precision point cloud generated by the DJI L1 LiDAR sensor and processed with the DJI Terra system. For the S3E dataset, since no ground-truth map is available and the ground-truth poses do not contain rotational information, we evaluate the geometric quality of the reconstructed maps using the average plane thickness and planarity metrics. Detailed definitions of the evaluation metrics are provided in MapEval \cite{mapeval}.

Table~\ref{table:mappingquality} presents the quantitative comparison of mapping quality between our method and LAMM. Our approach achieves the best performance across all datasets, including lower average Wasserstein distance (AWD), lower Chamfer distance (CD), smaller Spatial Consistency Score (SCS), and lower Mean Map Entropy (MME). These results indicate that our method produces higher-quality maps with better alignment to the ground-truth maps. In these three large-scale environments covering areas of over 100,000 square meters, our AWD ranges from 0.25 m to 0.65 m, and the CD ranges from 0.43 m to 1.11 m, demonstrating strong consistency with the ground-truth maps. Furthermore, the lower SCS and MME values suggest that our method achieves better spatial consistency than LAMM, which can be attributed to the proposed spatial BA and the LPGO for improving local accuracy and global consistency. Additionally, Fig.~\ref{fig:AWD} visualizes the Wasserstein distance distribution on the Island dataset. As shown in Fig.~\ref{fig:AWD} (a), the errors are concentrated within the $3 \sigma$ bound (0.93 m), with most voxels around 0.25 m. Fig.~\ref{fig:AWD} (b) shows that almost all voxels have Wasserstein distance values below 0.5 m and exhibit a spatially consistent distribution, indicating that the reconstructed multi-robot map maintains high spatial consistency.

Table~\ref{table:thickness} reports the averaged plane thickness and planarity evaluation results over all sequences in the S3E dataset. Our method achieves the best performance compared to baselines, indicating that our spatial BA significantly improves geometric accuracy and effectively mitigates the multi-robot map divergence commonly observed in PGO-based methods.

%%%%%%%%%%%    多机建图质量对比    %%%%%%%%%%%%%
\begin{table}[t]
\centering
\caption{Mapping Quality Evaluation}
\label{table:mappingquality}
\renewcommand{\arraystretch}{1.0}
\begin{threeparttable}
\begin{tabular}{@{}lccccc@{}}
\toprule[0.03cm]
\multirow{2}{*}{\textbf{Data}} & \multirow{2}{*}{\textbf{Method}} & \multicolumn{4}{c}{\textbf{Metrics}} \\
\cmidrule(lr){3-6}
& &  \textbf{AWD(m) $\downarrow$} & \textbf{CD(m) $\downarrow$} & \textbf{SCS $\downarrow$} &\textbf{MME $\downarrow$} \\
\midrule[0.03cm]
\multirow{2}{*}{Island}
& LAMM  & 1.30 & 6.44 & 0.50 & -5.51 \\
& Ours  & \textbf{0.30} & \textbf{0.43} & \textbf{0.49} & \textbf{-6.13} \\
\midrule
\multirow{2}{*}{Town} 
& LAMM & 1.77 & 10.64 & 0.57 & -6.37 \\
& Ours & \textbf{0.65} & \textbf{1.11} & \textbf{0.47} & \textbf{-6.52} \\
\midrule
\multirow{2}{*}{Airport} 
& LAMM & 0.75 & 36.05 & 0.62 & -6.17 \\
& Ours & \textbf{0.25} & \textbf{0.47} & \textbf{0.57} & \textbf{-6.28} \\
\bottomrule[0.03cm]
\end{tabular}
\end{threeparttable}
\vspace{-0.5em}
\end{table}

\vspace{20pt}

%%%%%%%%%%%%%%%%%  S3E 中的平面厚度  %%%%%%%%%%%%%%%%%%
\begin{table}[!t]
\centering
\caption{Geometry Accuracy Comparison}
\label{table:thickness}
    \begin{threeparttable}
        \begin{tabular}{lccc}
            \toprule
            Metric & DCL-SLAM & LAMM & Ours \\
            \midrule   
            Thickness(m) $\downarrow$ & 0.12 & 0.14 & \textbf{0.08} \\
            Planarity  $\uparrow$     & 0.61 & 0.59 & \textbf{0.86} \\
            \bottomrule
        \end{tabular}
    \end{threeparttable}
\end{table}

%%%%%%%%%%%%%%  推土距离的图   %%%%%%%%%%%%%%%%
\begin{figure}[!t]
  \centering
  \includegraphics[width=1.0\linewidth]{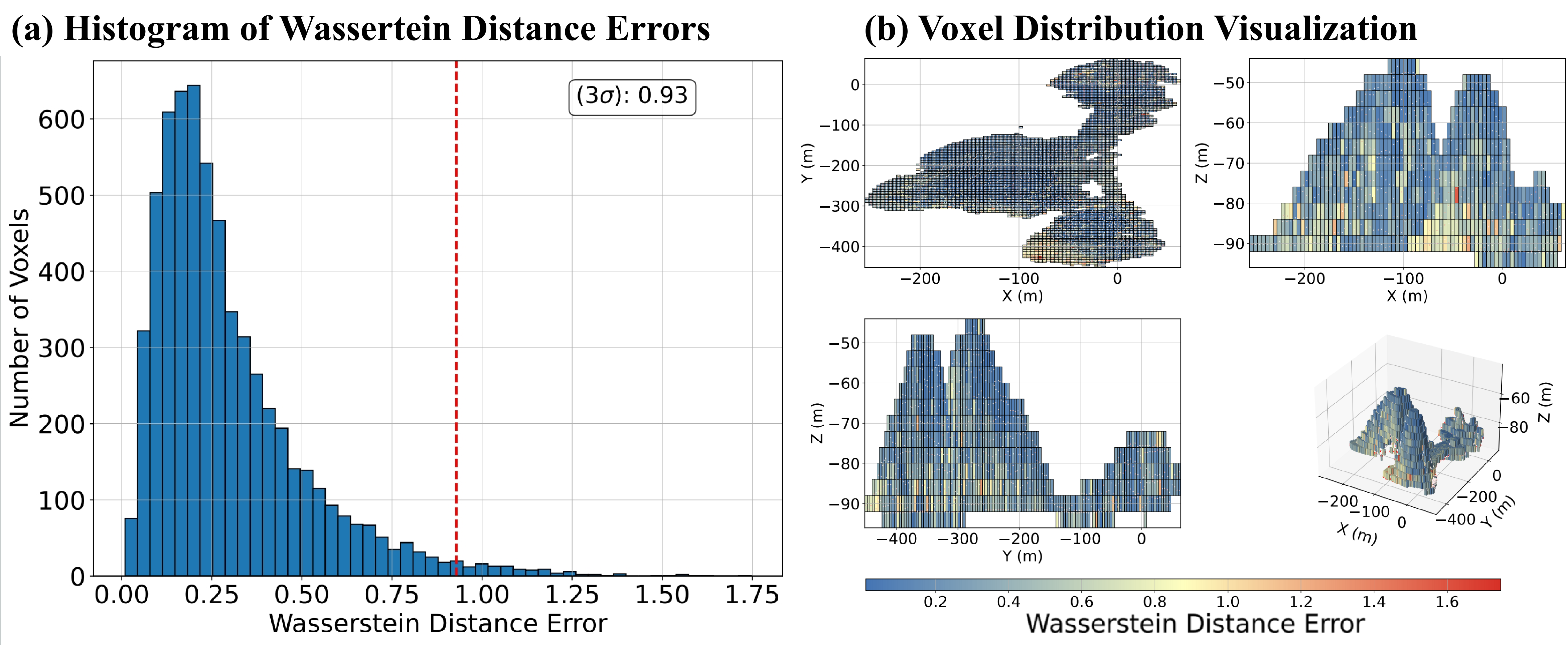}
  \caption{Visualization of Wasserstein distance errors on the Island dataset.
 (a) Histogram of errors over all voxels, where the dashed line indicates the $3 \sigma$ bound (0.93 m).
 (b) Spatial distribution of voxel-wise errors visualized from multiple coordinates, showing the consistency of reconstruction quality across the environment.}
  \label{fig:AWD}
\end{figure}

%%%%%%%%%%%%%%%%%%%%%%%%%%%%%%% 回环处理实验  %%%%%%%%%%%%%%%%%%%%%%%%%%%%%%%%%%%%%%%%
\subsection{\changed{Loop Processing Study}}
\label{subsection:loopstudy}

\changed{In this section, we dig deeper into the loop processing module in Section \ref{subsection:loopprocessing}. Thorough experiments are conducted on S3E dataset to demonstrate the effectiveness of the whole loop processing procedure.} 

\begin{table}[t]
\centering
\caption{Loop Closure Recall Comparison}
\label{table:loop_recall}
    \begin{threeparttable}
        \begin{tabular*}{0.48\textwidth}{@{ } @{\extracolsep{\fill}} c c c c c c @{ }}  
        \toprule
        \makecell{Sequence} & \makecell{Initial} & \makecell{After Rejection} &  \makecell{After Recalling} & \makecell{RMSE \\ $w/o$ LR} & \makecell{RMSE \\ $w/$ LR}\\
        \midrule
        Campus\_1 & 225 & 185 & 190  & 9.98 & \textbf{9.97} \\
        Dormitory & 106 & 76  & 77 & 4.15 & \textbf{4.12} \\
        Library & 286 & 235  & 238 & 1.41 & \textbf{1.40} \\
        Tunnel & 208 & 182  & 194  & 1.40 & \textbf{1.29} \\
        \bottomrule
        \end{tabular*}
    \end{threeparttable}
\end{table}

\begin{figure}[!t]
  \centering
  \includegraphics[width=0.96\linewidth]{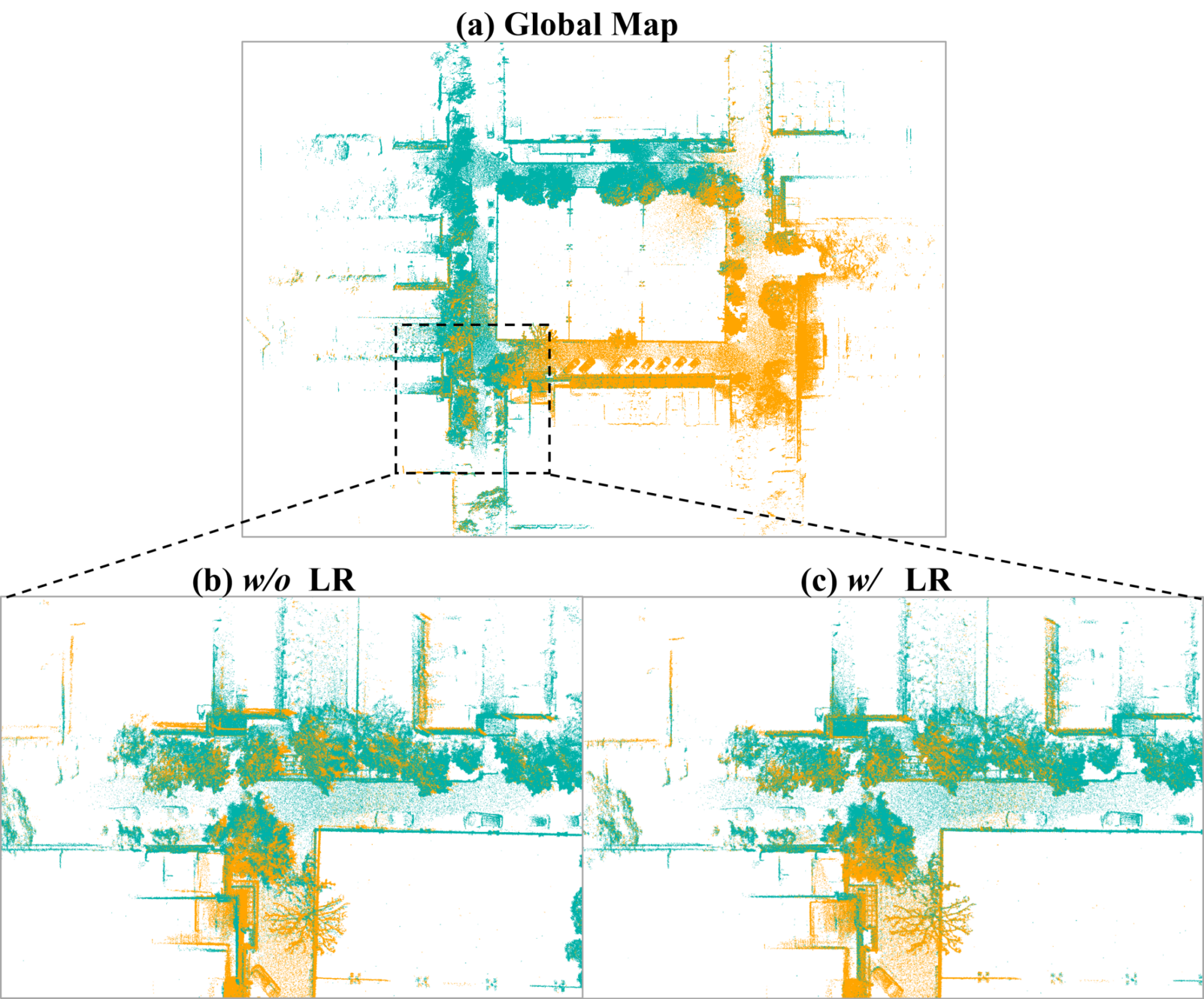}
  \caption{(a) shows the global map in Laboratory with loop recall. (b) and (c) present enlarged views of the same region without and with loop closure, respectively. In (c), the recalled loop information contributes to improved alignment and consistency across multi-robot maps.}
  \label{fig:looprecall}
\end{figure}

\changed{First, we highlight the role of the loop recall module. We compare the performance of the FPGO with and without loop recall (LR) using RMSE of the ATE (m), reported in the table as “RMSE $w/$ LR” and “RMSE $w/o$ LR,” respectively. As the S3E dataset contains a large number of complex scenes that combine unstructured and structured data, as well as many scenes with high similarity, it is easy to have a large number of false loops, so we use it as the dataset for loop recall.} 

\changed{Table \ref{table:loop_recall} shows the number of loops in different loop processing stages, as well as the results of the FPGO with and without loop recall. The experimental results show that the FPGO achieves smaller errors and better performance when mistakenly rejected loops are recalled. Beyond its impact on PGO, loop recall also introduces additional constraints that benefit subsequent bundle adjustment.}
\changed{Fig.~\ref{fig:looprecall} illustrates the results in our laboratory scene. (a) presents the global map with the loop recall step. (b) and (c) compare the local mapping results without and with loop recall, respectively. It is evident that the loop recall introduces additional constraints that significantly improve point cloud alignment within the selected region.}

\newcommand{\bestf}[1]{\textbf{#1}}
\newcommand{\badloss}[1]{\textcolor{red!70!black}{#1}}

\begin{table}[t]
  \centering
  \caption{Loop Closure Outlier Rejection Comparison.}
  \label{tab:loop_filter_comparison}
  \begin{tabular}{llcccc}
  \hline
  Dataset & Method & Loops & Precision & Recall & F1 \\
  \hline
  \multirow{5}{*}{Campus\_1}
  & \cellcolor{gray!10}Ours & \cellcolor{gray!10}190 & \cellcolor{gray!10}86.8\% & \cellcolor{gray!10}98.2\% & \cellcolor{gray!10}\bestf{92.2\%} \\
  & PCM Best-F1 & 209 & 78.0\% & 97.0\% & 86.5\% \\
  & GNC Best-F1 & 189 & 81.0\% & 91.1\% & 85.7\% \\
  & PCM P$\geq$Ours & -- & -- & -- & -- \\
  & GNC P$\geq$Ours & \badloss{125} & 87.2\% & \badloss{64.9\%} & 74.4\% \\
  \hline
  \multirow{5}{*}{Dormitory}
  & \cellcolor{gray!10}Ours & \cellcolor{gray!10}77 & \cellcolor{gray!10}87.0\% & \cellcolor{gray!10}94.4\% & \cellcolor{gray!10}\bestf{90.5\%} \\
  & PCM Best-F1 & 104 & 68.3\% & 100.0\% & 81.1\% \\
  & GNC Best-F1 & 89 & 74.2\% & 93.0\% & 82.5\% \\
  & PCM P$\geq$Ours & -- & -- & -- & -- \\
  & GNC P$\geq$Ours & \badloss{11} & 90.9\% & \badloss{14.1\%} & 24.4\% \\
  \hline
  \multirow{5}{*}{Library}
  & \cellcolor{gray!10}Ours & \cellcolor{gray!10}238 & \cellcolor{gray!10}89.5\% & \cellcolor{gray!10}95.1\% & \cellcolor{gray!10}\bestf{92.2\%} \\
  & PCM Best-F1 & 280 & 78.2\% & 97.8\% & 86.9\% \\
  & GNC Best-F1 & 259 & 83.4\% & 96.4\% & 89.4\% \\
  & PCM P$\geq$Ours & -- & -- & -- & -- \\
  & GNC P$\geq$Ours & \badloss{55} & 92.7\% & \badloss{22.8\%} & 36.6\% \\
  \hline
  \multirow{5}{*}{Tunnel}
  & \cellcolor{gray!10}Ours & \cellcolor{gray!10}194 & \cellcolor{gray!10}86.1\% & \cellcolor{gray!10}95.4\% & \cellcolor{gray!10}90.5\% \\
  & PCM Best-F1 & 195 & 84.1\% & 93.7\% & 88.6\% \\
  & GNC Best-F1 & 196 & 86.7\% & 97.1\% & \bestf{91.6\%} \\
  & PCM P$\geq$Ours & \badloss{3} & 100.0\% & \badloss{1.7\%} & 3.4\% \\
  & GNC P$\geq$Ours & 196 & 86.7\% & 97.1\% & 91.6\% \\
  \hline
  \end{tabular}
\end{table}

\changed{To evaluate the effectiveness of the proposed outlier rejection module, we compare it with two commonly used loop closure outlier rejection methods, namely PCM \cite{pcm1} and GNC \cite{gnc}. All methods take the same raw loop closure candidates from RING++ \cite{ring++} as input. Since both PCM and GNC are sensitive to parameter choices, we perform a parameter sweep for each baseline. For PCM, we vary the pairwise consistency threshold using 15 settings ranging from 0.02 to 50.0. For GNC, we vary the robust residual truncation threshold $\bar{c}$ using 9 settings ranging from 5 to 40.}

\changed{Table~\ref{tab:loop_filter_comparison} reports the loop rejection results. Since ground-truth labels for loop closure candidates are unavailable, we use the optimized trajectory as reference (obtained by full framework) and regard a candidate as correct if the Euclidean distance between the two associated poses is below 5 m. The retained loops are treated as predicted positives, from which the number of loops, precision, recall, and F1 score are computed. For PCM and GNC, "Best-F1" denotes the best result among all tested parameters, while "P$\geq$Ours" denotes the setting that preserves the most loops under the constraint that its precision is no lower than ours. Entries marked with "--" indicate that no tested setting satisfies this constraint. In the table, bold F1 scores highlight the best result among ours and the Best-F1 settings of PCM/GNC for each dataset, and colored entries highlight notable loop loss under the precision-constrained setting.}

\changed{The results show that our method achieves a more favorable balance between outlier rejection and inlier preservation. In Campus\_1, Dormitory, and Library, our method achieves higher F1 scores than both PCM and GNC under their Best-F1 settings, while in Tunnel it remains comparable to the best GNC result. More importantly, the P$\geq$Ours results show that, when constrained to reach the same precision level as our method, the baselines often retain substantially fewer loop closures, or fail to find a feasible parameter setting. For example, in Library, GNC retains only 55 loop closures under this constraint, compared with 238 retained by our method. These results indicate that PCM and GNC are more sensitive to the precision--recall trade-off, whereas the proposed loop processing module preserves a sufficient number of reliable loop closures while maintaining high precision, providing clean and sufficiently dense constraints for the PGO and BA stages.}

%%%%%%%%%%%%%%%%%%%%%%%%%%%%%%%%%%%%%%%%%%%%%%%%%%%%%%%%%%%%%%%%%%%%%%%%%%%%%%%%%%%%消融实验%%%%%%%%%%%%%%%%%%%%%%%%%%%%%%%
\subsection{Ablation Study}
\label{subsection:ablationstudy}

%多机消融实验表
\begin{table}[t]
\centering
\caption{RMSE of The ATE(m) of Ablation Study}
\label{table:ablation}
    \begin{threeparttable}
        \begin{tabular*}{0.48\textwidth}{@{\hspace{15pt}}@{\extracolsep{\fill}} c c c c @{\hspace{15pt}}}
        \toprule
        \makecell{Sequence} & \makecell{FPGO} & \makecell{FPGO + BA} & \makecell{LEMON Full}\\
        \midrule
        Campus\_1 & 9.97 & 10.05 & \textbf{9.94}  \\
        Campus\_3 & 3.82 & 4.26 & \textbf{3.51} \\
        Dormitory & 4.12 & 4.40 & \textbf{3.46} \\
        Tunnel & 1.29 & 1.52 & \textbf{0.98} \\
        Stairs & 0.194 & 0.191 & \textbf{0.142} \\
        Valley & 7.84 & 7.39 & \textbf{6.76} \\
        Island & 0.963 & 1.076 & \textbf{0.801} \\
        \bottomrule
        \end{tabular*}
    \end{threeparttable}
\end{table}

\begin{figure}[!t]
  \centering
  \includegraphics[width=0.96\linewidth]{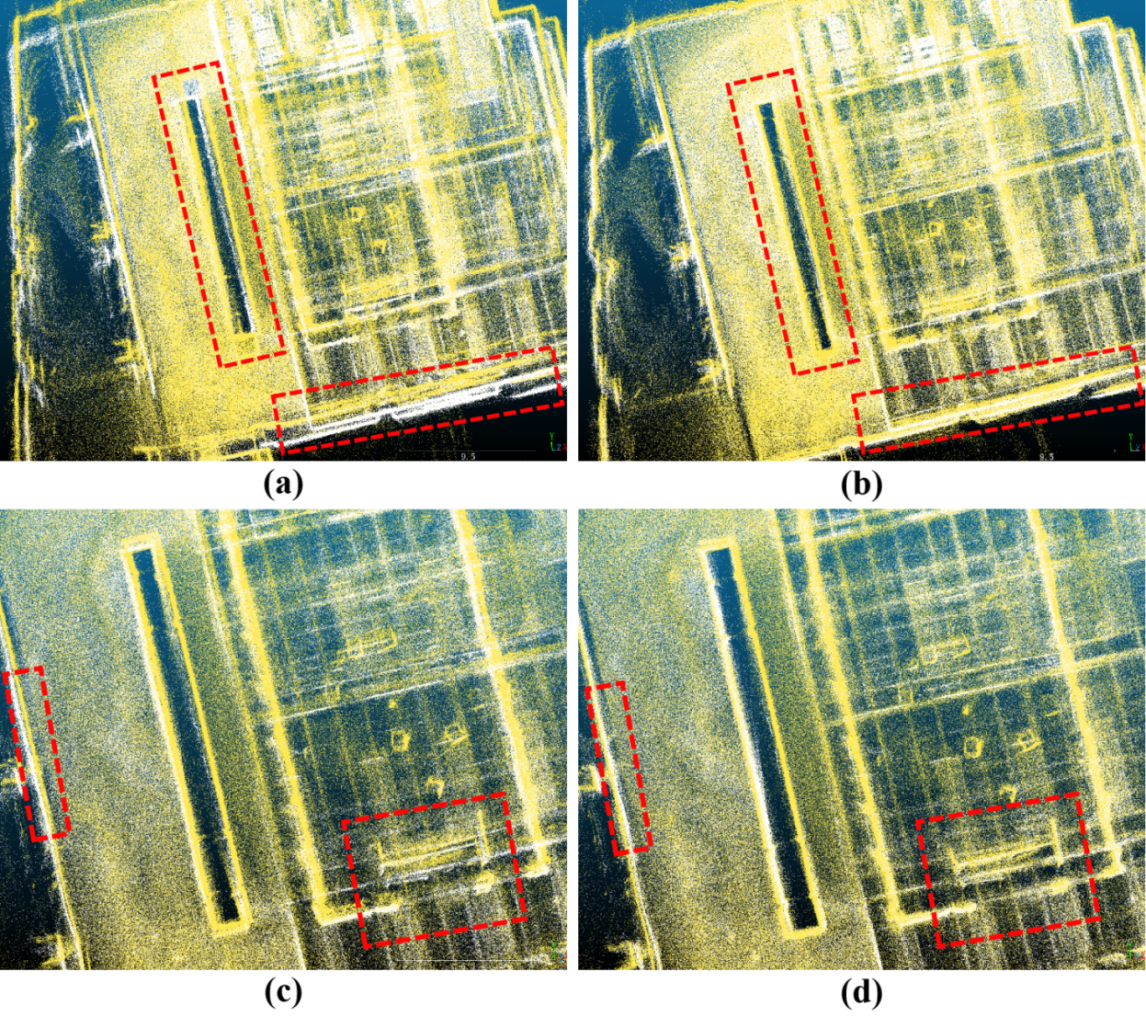}
  \caption{Result of ablation study in our Flying Arena dataset. (a) and (b) show the merged map of FPGO and FPGO + BA. (c) and (d) show the merged map of FPGO + BA and LEMON-mapping full model. The comparison regions are framed in red.}
  \label{fig:ablation}
\end{figure}

\begin{figure}[!t]
  \centering
  \includegraphics[width=1.0\linewidth]{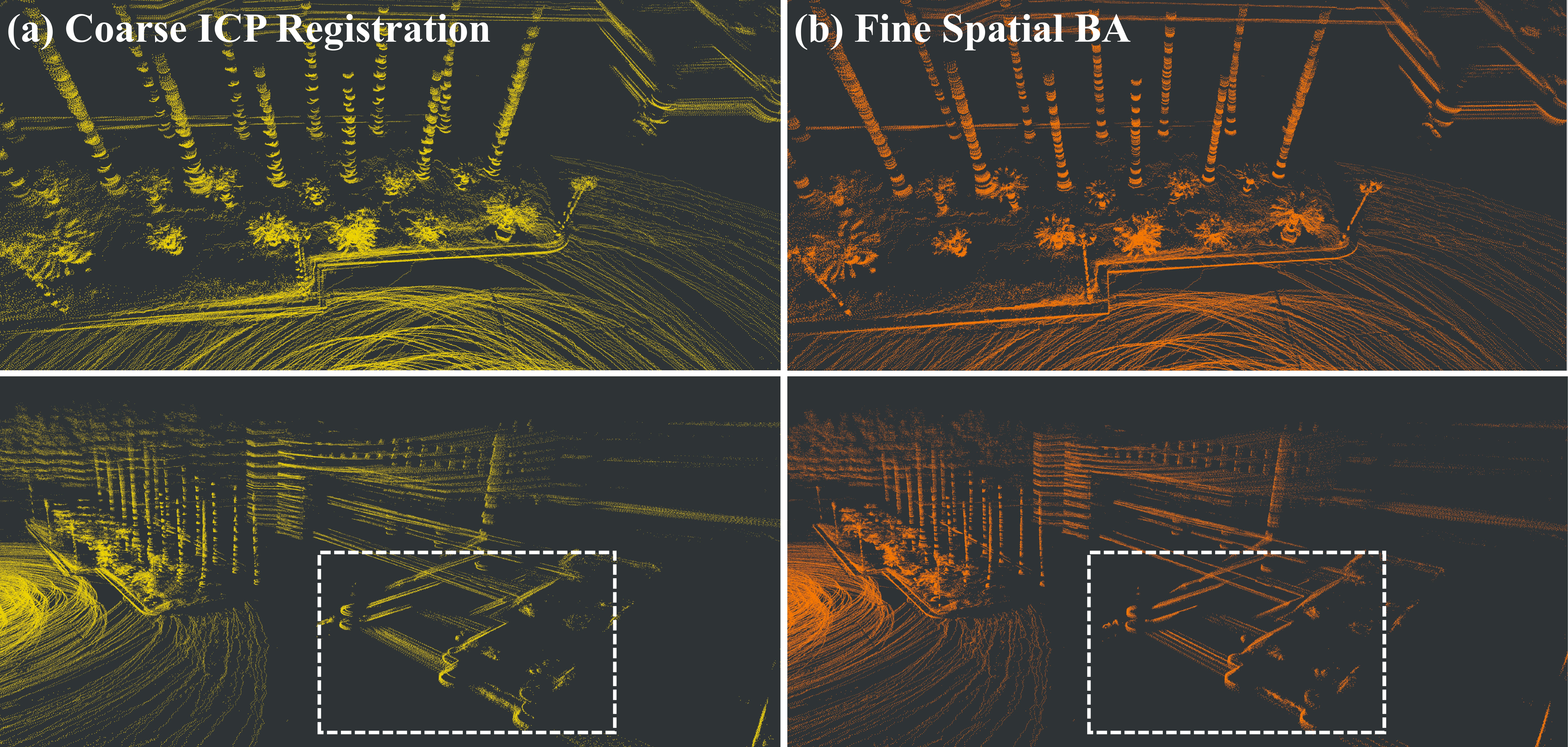}
  \caption{The map comparison of coarse registration (a) and fine multi-robot spatial BA (b) in S3E Library. It can be clearly seen that GICP only achieves basic alignment of multi-robot submaps. In the upper image, tree trunks and signs show obvious misalignment, and road edges are noticeably blurred. In the lower image, steps exhibit severe \changed{divergence}. In contrast, after fine spatial BA, the submaps are completely consistent, with excellent geometric quality and clear edge structures.}
  \label{fig:coarse-fine-compare}
\end{figure}

\begin{figure*}[!t]
  \centering
  \includegraphics[width=0.96\linewidth]{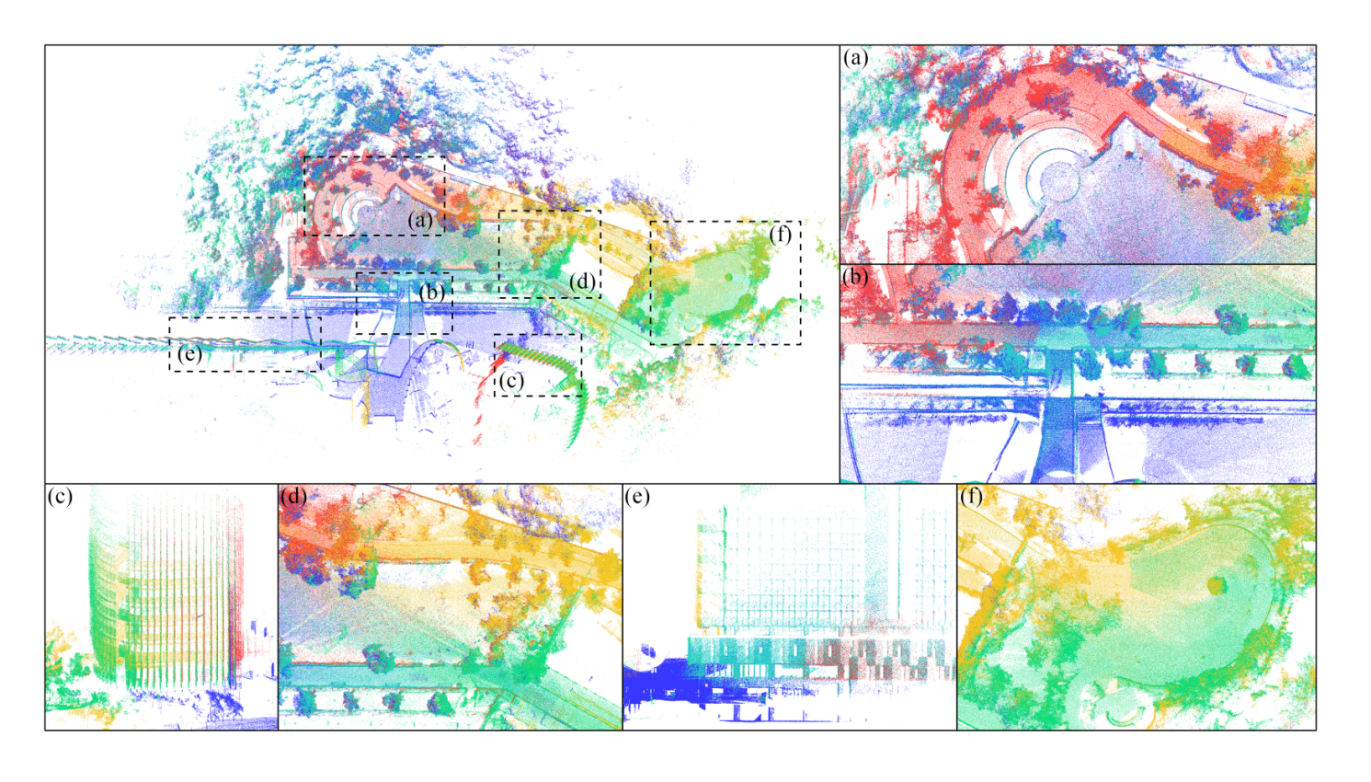}
  \caption{The map merging result of the five-session dataset in R$^3$LIVE HKU Park. The local maps of (a)-(f) are enlarged to show the details.}
  \label{fig:r3live}
\end{figure*}

To understand the contribution of each component in our framework, we perform ablation experiments on the S3E, GEODE, MARS-LVIG datasets, as well as our self-collected dataset. All sequences of our self-collected dataset are divided into two sessions. 
We evaluate the performance of three variants in our map merging module: (1) First Pose Graph Optimization only (FPGO), (2) First Pose Graph Optimization and Spatial Bundle Adjustment (FPGO + BA) and (3) the full LEMON-Mapping system (LEMON Full). The RMSE of the ATE (m) is used for quantitative comparison.

Table~\ref{table:ablation} summarizes the results, with the best values highlighted in bold. The complete framework consistently achieves the lowest RMSE, validating the effectiveness of combining spatial BA and two PGO steps. For the variant of FPGO + BA, although local BA improves relative accuracy, it may disrupt the continuity of odometry since it only refines loop regions. Consequently, its standalone use can increase RMSE compared to FPGO alone. However, the last PGO integrates both local BA constraints and odometry continuity, propagating local accuracy to the global map and significantly reducing RMSE.

Fig.~\ref{fig:ablation} shows the map merging results of three variants above in our Flying Arena scene. (a) and (b) demonstrate the function of spatial BA to improve local consistency and accuracy compared to PGO only. (c) and (d) in Fig.~\ref{fig:ablation} further illustrate that maps refined with both BA and last pose graph optimization (LPGO) have better global consistency than those using BA alone, confirming the role of LPGO in maintaining global map structure.

In particular, we visualize the comparison between coarse GICP registration and the proposed fine spatial BA. Fig.~\ref{fig:coarse-fine-compare} (a) shows the result of GICP-based coarse alignment, where only basic submap alignment is achieved, and noticeable misalignments appear in tree trunks, traffic signs, and road boundaries. Fig.~\ref{fig:coarse-fine-compare} (b) shows the result after applying the proposed spatial BA, where submaps become fully consistent with clear edge structures and significantly improved geometric quality. This confirms that spatial BA is a critical component for high-precision multi-robot mapping. While GICP only ensures coarse submap-level alignment, the proposed spatial BA performs fine-grained joint optimization at the pose level, enabling globally consistent and locally accurate map reconstruction. As a result, the originally divergent multi-robot point clouds are refined to achieve highly consistent and precise reconstruction, approaching the geometric fidelity typically observed in short-range single-robot mapping.

%%%%%%%%%%%%%%%%%%%%%%%%%%%%%%%%%%%%%%%%%%%%%%%%%%%%%%%%%%%%%%%%%%%%%%%%%%%%%%%%%%%%扩展性实验%%%%%%%%%%%%%%%%%%%%%%%%%%%%%%
\subsection{Scalability Study}
\label{subsection:scalabilitystdudy}

\begin{table}[t]
\centering
\caption{Scalability Study Results}
\label{table:scalability}
    \begin{threeparttable}
        \begin{tabular*}{0.48\textwidth}{@{\hspace{10pt}} @{\extracolsep{\fill}} c c c c @{\hspace{10pt}}}  
        \toprule
        \makecell{Sequence} & 
        \makecell{Robot Number} &
         \makecell{Merged Number} &
         \makecell{Success Rate} \\
        \midrule
        HKU Park & 5 & 5 & \textbf{100\%} \\
        HKU Campus & 10 & 10 & \textbf{100\%} \\
        HKUST Campus & 20 & 20 & \textbf{100\%} \\
        \bottomrule
        \end{tabular*}
    \end{threeparttable}
\end{table}

Most existing map fusion frameworks are limited to scenarios that involve only a few robots (less than five in \cite{LAMM, segmap, dclslam, discoslam}), and their scalability to large-scale deployments is not yet proven. To assess the scalability of our proposed framework, we perform experiments on the R$^3$LIVE dataset, which we subdivide into groups of 5, 10, and 20 sessions. A successful fusion is defined as the correct alignment of each session with all its adjacent sessions that share sufficient map overlap.

Table~\ref{table:scalability} summarizes the corresponding results of three cases. Our framework achieves a success rate 100\% in all experiments, including the five-session, ten-session, and twenty-session scenarios. These results demonstrate the scalability and robustness of the LEMON-Mapping system when dealing with large numbers of multi-robot map merging.

Fig.~\ref{fig:r3live} visualizes the fused map for the five-session case of HKU Park in R$^3$LIVE dataset, including a global point cloud map from a bird’s eye view (BEV) and 6 enlarged map details. The five-session map shows good fusion effect and local accuracy in both structured and unstructured environments.

%%%%%%%%%%%%%%%%%%%%%%  运行时间和内存  %%%%%%%%%%%%%%%%%%%%%%%%%%%
\subsection{\changed{Runtime and Memory Efficiency Analysis}}
\label{subsec:efficiency_analysis}

\changed{To evaluate computational efficiency and scalability, we compare LEMON-Mapping across five S3E sequences against HBA \cite{hba} and a global full BALM2 \cite{balm2} pipeline.}

\changed{As illustrated in Fig.~\ref{fig:runtime} (middle), our execution time is primarily governed by the number of loop closures rather than the total trajectory length. For instance, the runtime increases in loop-dense sequences such as Campus\_3 (238 loops) and Library (191 loops) due to the frequent activation of local spatial BA windows. Conversely, the Dormitory sequence (62 loops) requires only $\sim$70 seconds to complete. Fig.~\ref{fig:runtime} (top) shows that the computational budget is heavily dominated by HBA ($\sim$67\%--80\%) and cluster preprocessing ($\sim$17\%--28\%, which involves PCA reordering and GICP coarse alignment). This distribution aligns well with our design of focusing computational effort on resolving multi-robot misalignments within critical overlapping regions. In contrast, Isolated DBA and global PGO alignment are highly fast, each consuming less than 5\% of the total runtime.}

\changed{
The results of memory efficiency are demonstrated in the bottom plot of Fig.~\ref{fig:runtime}. LEMON-Mapping maintains a lightweight and bounded memory usage consistently across all sequences, outperforming even the HBA baseline. In contrast, the "Full BA (fused)" pipeline demands a massive 30--50~GiB of memory due to the scale of the global Hessian matrix. Note that executing a conventional global BALM2 at the raw 10Hz rate invariably triggers Out-Of-Memory (OOM) failures. To enable a feasible baseline comparison, it must aggressively aggregate $f$ frames (e.g., $f=30$) into a single submap. This prohibitive memory of full BA further underscores the lightweight nature of our method in terms of memory efficiency.
}

\begin{figure}[!t]
  \centering
  \includegraphics[width=0.96\linewidth]{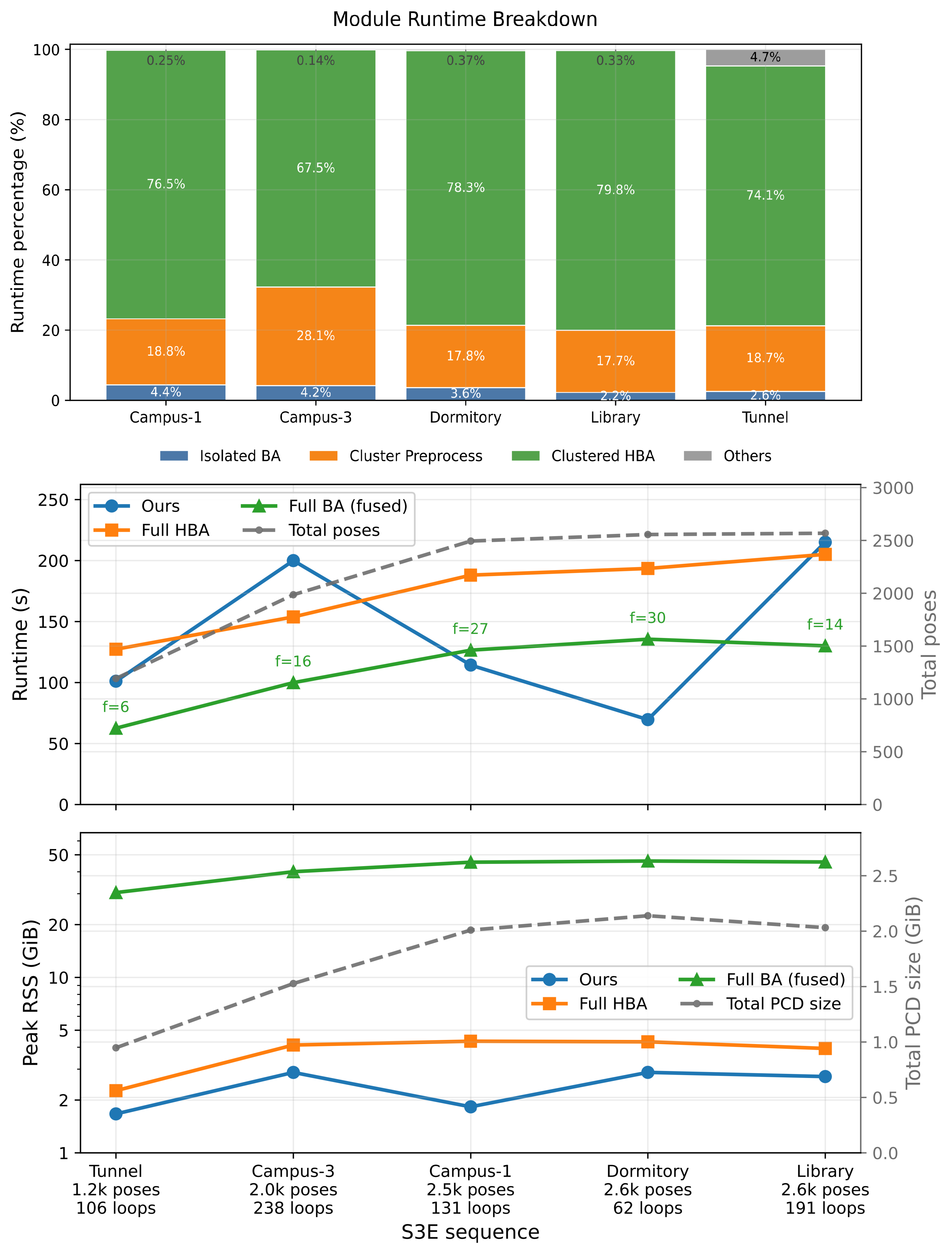}
  \caption{Comprehensive efficiency evaluation on S3E datasets.
  (Top) Percentage runtime breakdown of individual modules within our framework. 
  (Middle) Computational runtime comparison across different sequences, with the total number of poses, loop closures, and the Full BA downsampling factor $f$ explicitly annotated for each dataset. 
  (Bottom) Peak memory consumption comparison.}
  \label{fig:runtime}
\end{figure}

\section{Conclusion and Future Work}
This paper presents \textbf{LEMON-Mapping}, a \textbf{L}oop-\textbf{E}nhanced Large-Scale Multi-Session Point Cloud Map \textbf{M}erging and \textbf{O}ptimizatio\textbf{N} framework for Globally Consistent Mapping. LEMON-Mapping is a scalable framework for a large number of robots. 
% Unlike existing methods that rely solely on pose sequences, it enhances the role of loop closures by directly utilizing geometric constraints from overlapping point clouds and reasonably transferring the local accuracy to the global map. The framework includes a loop processing module for reliable loop selection and innovate loop recall, and a map merging module that combines two-step pose graph optimization with a spatial bundle adjustment tailored for multi-robot systems. The spatial bundle adjustment operates on loop-based spatial windows, allowing geometry-aware pose optimization across unordered and cross-robot trajectories. To ensure global consistency, we sparsify the BA constraints and propagate local accuracy across the entire map using a refined pose graph optimization strategy. 
% Experimental results in S3E, MARS-LVIG, GEODE, R$^3$LIVE and our self-collected dataset demonstrate that our approach maintains high map merging success rates and robustness as the number of sessions increases, while outperforming traditional methods in geometric accuracy and scalability. 
Unlike existing pose-only methods, it utilizes geometric constraints from overlapping point clouds for accuracy transfer. It features a loop processing module for reliable selection and recall, and a map merging module driven by two-step pose graph optimization and a window-based spatial bundle adjustment. This architecture enables geometry-aware, globally consistent optimization across unordered cross-robot trajectories.
\changed{Future work will focus on two main directions. First, to further improve the accuracy of multi-session mapping, we plan to develop a principled, quantitative covariance estimation strategy through probabilistic modeling of point cloud residuals, deriving dimensionally consistent information matrices to replace current heuristic weights. Second, leveraging the exceptionally low memory footprint of our localized Spatial BA, we plan to introduce a multi-threaded parallel computing architecture. By optimizing independent spatial windows simultaneously, we aim to drastically reduce the overall runtime and achieve a highly optimal time-space trade-off for massive-scale deployments.}

\section*{Appendix: PROOF OF LEMMA}
\label{appendix}
% \subsection{Proof of Lemma \ref{lemma:Sym2Continuos}}
% \label{appendix:lemma1}
% \begin{proof}
    
% Since each $\mathbf A(x)$ is real symmetric, the spectral theorem guarantees an orthogonal matrix $\mathbf Q(x)$ and a real diagonal matrix
% \[
% \mathbf\Lambda(x)
% =\operatorname{diag}\bigl(\lambda_1(x),\dots,\lambda_n(x)\bigr)
% \]
% such that
% \[
% \mathbf A(x)
% =\mathbf Q(x)\,\mathbf\Lambda(x)\,\mathbf Q(x)^\top.
% \]
% In particular, the eigenvalues $\lambda_i(x)$ are all real.

% Consider the characteristic polynomial
% \[
% p_x(t)
% =\det\bigl(\mathbf A(x)-t\,\mathbf I\bigr).
% \]
% On one hand, by orthogonal diagonalization,
% \[
% p_x(t)
% =\det\bigl(\Lambda(x)-t\,I\bigr)
% =\prod_{i=1}^n\bigl(\lambda_i(x)-t\bigr).
% \]
% On the other hand, expanding the determinant in terms of the entries of $\mathbf A(x)$ shows
% \[
% p_x(t)
% =t^n + c_{n-1}(x)\,t^{n-1} + \cdots + c_1(x)\,t + c_0(x),
% \]
% where each coefficient $c_k(x)$ is a finite sum of products of the entries of $\mathbf A(x)$.  Since $\mathbf A(x)$ is continuous in $x$, each $c_k(x)$ is a continuous function of $x$.

% By the fundamental theorem of algebra and the fact that all roots of $p_x$ are real, we may label them in nondecreasing order
% \[
% \lambda_1(x)\;\le\;\lambda_2(x)\;\le\;\cdots\;\le\;\lambda_n(x).
% \]
% A standard result on the continuous dependence of polynomial roots on coefficients now implies that each ordered root $\lambda_i(x)$ varies continuously with the coefficients $\{c_k(x)\}$, and hence with $x$ itself.  Therefore, every eigenvalue function $x\mapsto\lambda_i(x)$ is continuous.
% \end{proof}

\subsection{Proof of Lemma \ref{lemma:compute_DBA}}
\label{appendix:lemma2}
\begin{proof}

The complexity of joint BA over all \(M\) poses is
\[
\mathcal{C}_{\mathrm{joint}}
= O\!\left(M_f M + M_f M^2 + M^3\right),
\]
where the three terms correspond to evaluating feature residuals, forming the 
Hessian, and solving the LM linear system.

Expanding the squared and cubic terms via group sizes \(m_i\) yields
\[
M^2 = \Big(\sum_i m_i\Big)^2
    = \sum_i m_i^2 + 2\sum_{i<j} m_i m_j,
\]
\[
M^3 = \Big(\sum_i m_i\Big)^3
    = \sum_i m_i^3 + \text{(mixed cross-group terms)}.
\]
Thus joint BA includes not only the per-group contributions \(\sum_i m_i^2\) and 
\(\sum_i m_i^3\), but also all cross-group interaction terms such as 
\(m_i m_j\), \(m_i^2 m_j\), and \(m_i m_j m_k\), which dominate when multiple 
groups have comparable size.

In DBA, only the poses in the current diffusion group \(i\) remain active while 
all inner groups are frozen. The BA performed in diffusion step \(i\) therefore 
has complexity
\[
O\!\left(M_f m_i + M_f m_i^2 + m_i^3\right).
\]
Summing across all diffusion steps gives the total DBA cost
\[
\mathcal{C}_{\mathrm{DBA}}
= O\!\left(M_f\sum_i m_i \;+\; M_f\sum_i m_i^2 \;+\; \sum_i m_i^3\right),
\]
which eliminates all cross-group terms present in 
\(\mathcal{C}_{\mathrm{joint}}\). Since
\[
\sum_i m_i^2 \le M^2, \qquad \sum_i m_i^3 \le M^3,
\]
DBA is never more expensive than joint BA.

Moreover, when groups are approximately balanced with \(m_i \approx M/D\), we obtain
\[
\sum_i m_i^2 \approx \frac{M^2}{D}, \qquad 
\sum_i m_i^3 \approx \frac{M^3}{D^2},
\]
showing that the dominating quadratic and cubic terms are reduced by factors of 
order \(D\) and \(D^2\), respectively. This proves that DBA yields a strictly 
lower computational cost and achieves substantial practical speedups.
\end{proof}

\subsection{Proof of Lemma \ref{lemma:covariance}}
\label{appendix:lemma3}
\begin{proof}
We proceed in two steps.

\subsubsection{Block-extraction lower bound for the measurement noise term.}
\label{paragraph:bound}
For each measurement cluster the per-measurement contribution is
\[
\mathbf L_{ij}\,\boldsymbol{\Sigma}_{c_{f_{ij}}}\,\mathbf L_{ij}^\top
=
\begin{bmatrix}
{^0\mathbf L_{ij}}\Sigma {^0\mathbf L_{ij}}^\top &
{^0\mathbf L_{ij}}\Sigma {^1\mathbf L_{ij}}^\top \\[4pt]
{^1\mathbf L_{ij}}\Sigma {^0\mathbf L_{ij}}^\top &
{^1\mathbf L_{ij}}\Sigma {^1\mathbf L_{ij}}^\top
\end{bmatrix},
\]
which is positive semidefinite because it is of the form $A\Sigma A^\top$. Subtracting the block-diagonal matrix that keeps only the bottom-right block yields
\[
\mathbf L_{ij}\Sigma\mathbf L_{ij}^\top
-
\begin{bmatrix} \mathbf 0 & \mathbf 0 \\[3pt] \mathbf 0 & {^1\mathbf L_{ij}}\Sigma {^1\mathbf L_{ij}}^\top \end{bmatrix}
=
\begin{bmatrix}
{^0\mathbf L_{ij}}\Sigma {^0\mathbf L_{ij}}^\top &
{^0\mathbf L_{ij}}\Sigma {^1\mathbf L_{ij}}^\top \\[4pt]
{^1\mathbf L_{ij}}\Sigma {^0\mathbf L_{ij}}^\top & \mathbf 0
\end{bmatrix}.
\]
The right-hand side can be written as
\[
\begin{bmatrix} {^0\mathbf L_{ij}} \\[2pt] \mathbf 0 \end{bmatrix}
\Sigma
\begin{bmatrix} {^0\mathbf L_{ij}} \\[2pt] \mathbf 0 \end{bmatrix}^\top
\succeq 0,
\]
hence
\[
\mathbf L_{ij}\Sigma\mathbf L_{ij}^\top
\succeq
\begin{bmatrix} \mathbf 0 & \mathbf 0 \\[3pt] \mathbf 0 & {^1\mathbf L_{ij}}\Sigma {^1\mathbf L_{ij}}^\top \end{bmatrix}.
\]
Summing over all measurements preserves the PSD ordering, and extracting the bottom-right block gives the key inequality
\begin{equation}\label{eq:block-noise-ineq}
\Big[\sum_{i,j}\mathbf L_{ij}\,\boldsymbol{\Sigma}_{c_{f_{ij}}}\,\mathbf L_{ij}^\top\Big]_{11}
\succeq
\sum_{i,j\in\text{group1}} {^1\mathbf L_{ij}}\,\boldsymbol{\Sigma}_{c_{f_{ij}}}\,{^1\mathbf L_{ij}}^\top.
\end{equation}

\subsubsection{Schur-complement ordering for Hessian inverses.}
\label{paragraph:schur}
Write the Schur complement of $\mathbf H_{00}$ in $\mathbf H$:
\[
\mathbf S \triangleq \mathbf H_{11} - \mathbf H_{10}\mathbf H_{00}^{-1}\mathbf H_{01}.
\]
Because $\mathbf H_{10}\mathbf H_{00}^{-1}\mathbf H_{01}\succeq 0$, we have
\[
\mathbf S \preceq \mathbf H_{11}.
\]
For two strict positive definite matrices $\mathbf A\preceq\mathbf B$ it follows that $\mathbf B^{-1}\preceq\mathbf A^{-1}$. Applying this to $\mathbf S\preceq\mathbf H_{11}$ yields
\[
\mathbf H_{11}^{-1}\preceq \mathbf S^{-1}.
\]

\subsubsection{Combine \ref{paragraph:bound} and \ref{paragraph:schur}.}
The bottom-right block of the joint covariance can be written using the Schur complement inverse $\mathbf S^{-1}$:
\[
\boldsymbol{\Sigma}^{\mathrm{joint}}_{11}
= \mathbf S^{-1}\,
\Big[\sum_{i,j}\mathbf L_{ij}\,\boldsymbol{\Sigma}_{c_{f_{ij}}}\,\mathbf L_{ij}^\top\Big]_{11}
\,\mathbf S^{-T}.
\]
Using \eqref{eq:block-noise-ineq} and the inverse-ordering $\mathbf H_{11}^{-1}\preceq\mathbf S^{-1}$ we obtain
\[
\begin{aligned}
\boldsymbol{\Sigma}^{\mathrm{joint}}_{11}
&= \mathbf S^{-1}\Big[\sum_{i,j}\mathbf L_{ij}\Sigma\mathbf L_{ij}^\top\Big]_{11}\mathbf S^{-T} \\
&\succeq \mathbf S^{-1}\Big(\sum_{i,j\in\text{group1}} {^1\mathbf L_{ij}}\Sigma {^1\mathbf L_{ij}}^\top\Big)\mathbf S^{-T} \\
&\succeq \mathbf H_{11}^{-1}\Big(\sum_{i,j\in\text{group1}} {^1\mathbf L_{ij}}\Sigma {^1\mathbf L_{ij}}^\top\Big)\mathbf H_{11}^{-T} \\
&= \boldsymbol{\Sigma}^{\mathrm{DBA}}_{1},
\end{aligned}
\]
where the second PSD inequality follows from left- and right-multiplying the PSD matrix
\(\sum_{i,j\in\text{group1}} {^1\mathbf L_{ij}}\Sigma {^1\mathbf L_{ij}}^\top\)
by $\mathbf S^{-1}$ and noting $\mathbf H_{11}^{-1}\preceq\mathbf S^{-1}$.

Thus \(\boldsymbol{\Sigma}^{\mathrm{DBA}}_{1}\preceq\boldsymbol{\Sigma}^{\mathrm{joint}}_{11}\), which proves the stated ordering.
\end{proof}

%%%%%%%%%%%%%%%%%%%%%%%%%%%%%%%%%%%%%%%%%%%%%%%%%%%%%%%%%%%%%%%%%%%%%%%%%%%%%%%

% \newpage
% \clearp

\bibliographystyle{IEEEtran}
\bibliography{ref}

\vfill

\end{document}